%% file: iclr2027_conference.tex
\documentclass{article} 
\usepackage{iclr2027_conference,times}

\input{math_commands.tex}

\usepackage{bibunits}
\usepackage{amsthm}
\usepackage{algorithm}
\usepackage{algpseudocode}
\usepackage{hyperref}
\usepackage{url}
\usepackage{graphicx}
\usepackage[labelformat=simple]{subcaption}

\usepackage{wrapfig}
\usepackage{enumitem}
\usepackage[table]{xcolor}
\usepackage{booktabs}
\usepackage{prompt_style}
\usepackage{marvosym}

\newtheorem{proposition}{Proposition}

\title{Multi-Agent System Search via Active Substructure-aware Policy Optimization}
\author{}
\author{
  Beicheng Xu$^{\dagger}$ \quad
  Bowen Fan$^{\dagger}$ \quad
  Weitong Qian$^{\dagger}$ \quad
  Lingching Tung$^{\dagger}$ \quad
  Bin Cui$^{\dagger}$\textsuperscript{\Letter} \\
  $^{\dagger}$School of CS, Peking University, Beijing, China\\
  beichengxu@stu.pku.edu.cn, bin.cui@pku.edu.cn\\
  \Letter\ Corresponding author
}

\iclrfinalcopy

\begin{document}

\maketitle
\fancyhead[L]{Preprint}
\fancyhead[R]{}

\begin{abstract}
LLMs enable multi-agent systems (MAS) to tackle complex tasks, but manually designing agent roles, prompts, and communication structures requires substantial expertise and effort.
This motivates learning policies that construct query-specific MAS from execution reward.
Existing approaches typically train these policies by repeatedly traversing a fixed set of training queries and assigning rewards at the workflow level.
However, this overlooks differences in queries' evolving learning potential and obscures which substructures improve solution quality.
In this paper, we propose Active Substructure-aware Policy Optimization (ASPO), a RL framework for query-level MAS search.
ASPO introduces an Adaptive Query-Selection Mechanism (AQSM) that focuses training on queries at the policy's competence boundary: those it can solve but not yet reliably.
A complementary discovery mechanism widens architectural exploration for hard queries, helping distinguish insufficient exploration from operator capability limits.
Beyond query selection, ASPO introduces substructure-level rewards that measure output-quality gains within each action's descendant subgraph.
These rewards guide proximal policy optimization to reinforce useful architectural refinements and discourage redundant or harmful computation.
Together, these mechanisms prioritize learnable queries and provide fine-grained feedback for learning effective MASs.
Across six benchmarks spanning mathematical reasoning, general question answering, and code generation, ASPO ranks first on every benchmark against twelve baselines.
\end{abstract}

\input{tex/introduction}
\input{tex/problem_formulation}
\input{tex/method}
\input{tex/experiment}
\input{tex/conclusion}

\section*{AI use statement}
We used generative AI tools to polish the manuscript's wording and assist with writing scripts for experimental analysis.
Three authors have reviewed the language revisions to ensure that they preserved the intended meaning and technical accuracy.
AI-assisted scripts were inspected and tested for correctness by three authors, and their outputs were checked against the experimental records.
We take full responsibility for the final content of this work, including all AI-assisted text and code.

\section*{Ethics statement}
Our work studies automated multi-agent system design through reinforcement learning.
Given this, we foresee no issues regarding fairness, privacy, or security, or any other harmful societal or ethical implications outside broader considerations for the field itself.
However, we do note that reliance on expensive LLMs may still create financial barriers to access.

\section*{Reproducibility statement}
We provide ASPO code (\texttt{aspo/aspo\_code/}), benchmark splits (\texttt{aspo/aspo\_data/}), and installation, data preparation, training, and evaluation instructions in the supplementary material.
Appendix~\ref{app:benchmarks} documents dataset sources and curation, while Appendices~\ref{app:baseline-setups} and~\ref{app:hyperparameters} detail baseline configurations and ASPO implementation settings.
The training procedure is summarized in Appendix~\ref{app:aspo-algorithm}, with the operator pool specified in Appendix~\ref{app:operator-space}.
The assumptions and proofs underlying our query-selection analysis are provided in Appendix~\ref{app:query-selection-theory}.

\bibliographystyle{iclr2027_conference}
\bibliography{iclr2027_conference}

\clearpage
\appendix
\begin{bibunit}[iclr2027_conference]
\makeatletter
\let\ASPO@natlinkstart\hyper@natlinkstart
\let\ASPO@natlinkbreak\hyper@natlinkbreak
\let\ASPO@natanchorstart\hyper@natanchorstart
\def\hyper@natlinkstart#1{%
  \ASPO@natlinkstart{appendix.#1}}
\def\hyper@natlinkbreak#1#2{%
  \ASPO@natlinkbreak{#1}{appendix.#2}}
\def\hyper@natanchorstart#1{%
  \ASPO@natanchorstart{appendix.#1}}
\makeatother

\input{appendix/related_work.tex}
\input{appendix/limitations.tex}

\section{Technical Details}
\input{appendix/state_encoder.tex}
\input{appendix/semantic_similarity.tex}
\input{appendix/query_selection_theory.tex}
\input{appendix/algorithm.tex}

\section{Experimental Setup}
\input{appendix/benchmarks.tex}
\input{appendix/representative_benchmarks.tex}

\input{appendix/baseline_setups.tex}

\input{appendix/hyperparameters.tex}
\input{appendix/cost_analysis.tex}
\input{appendix/architecture_sensitivity.tex}

\input{appendix/operator_space.tex}

\renewcommand{\refname}{Appendix References}
\putbib[iclr2027_conference]
\end{bibunit}

\end{document}

%% file: math_commands.tex
\usepackage{amsmath,amsfonts,bm}

\def\eqref#1{equation~\ref{#1}}

\def\1{\bm{1}}

\DeclareMathAlphabet{\mathsfit}{\encodingdefault}{\sfdefault}{m}{sl}
\SetMathAlphabet{\mathsfit}{bold}{\encodingdefault}{\sfdefault}{bx}{n}



%% file: tex/introduction.tex
\section{Introduction}
\label{sec:intro}


Multi-agent systems (MAS) built on large language models (LLMs) combine complementary agent capabilities through communication~\citep{MetaGPT_hong2024metagpt}.
However, manually designing their roles, prompts, and communication structures requires substantial expertise and repeated trial and error~\citep{ADAS_hu2025automated,AFlow_zhang2025aflow,AgentSquare_shang2025agentsquare}.
In response, researchers have explored automating individual aspects of agentic systems, including prompt optimization~\citep{DSPy_khattab2024dspy,TextGrad_yuksekgonul2025optimizing} and multi-agent collaboration~\citep{AgentVerse_chen2024agentverse,DyLAN_liu2024a}.
More broadly, recent work has extended automation to the design of complete MASs.
One prominent direction is \textbf{task-level MAS search}, which seeks an effective MAS that can be reused across queries within a target task domain~\citep{GPTSwarm_pmlr-v235-zhuge24a,ADAS_hu2025automated,AgentSquare_shang2025agentsquare,BOAD_ICLR2026_0d89023e}.
For example, AFlow~\citep{AFlow_zhang2025aflow} represents workflows using code, and applies Monte Carlo Tree Search to guide their optimization.
MASS~\citep{MASS_ICLR2026_1ab4e0e8} optimizes agent prompts and searches a pruned space of workflow topologies.
EvoMAS~\citep{EvoMASHiL_wei2026evomas} uses evolutionary search to optimize workflow structures, agent roles, and prompts.
However, task-level search follows a largely ``one-size-fits-all'' design objective, seeking a single MAS for an entire task domain rather than tailoring the architecture to each query.

To address this limitation, \textbf{query-level MAS search} adapts architectural decisions to individual queries rather than deploying a single shared MAS.
Existing approaches can be grouped by how they construct query-specific MASs:
(i) The first category constructs MAS by selecting combinations of agents from a predefined pool~\citep{RAAS_yang2026raas}.
MaAS~\citep{MaAS_pmlr-v267-zhang25bi} uses an RL-trained controller to select operators layer by layer from an agentic supernet.
Puppeteer~\citep{Puppter_NEURIPS2025_f1320d2e} instead conditions each decision on the query and accumulated intermediate outputs.
(ii) The second category adapts communication graphs over a predefined agent pool~\citep{CARD_ICLR2026_fdc05db3,RADAR_zhang2026radar,Selforg_ICLR2026_43833a8a}.
G-Designer~\citep{G-Designer_pmlr-v267-zhang25cu} uses a variational graph autoencoder to generate communication graphs, 
whereas MAGE~\citep{MAGE_zhao2026navigating} uses energy-guided diffusion to generate communication graphs.
(iii) The third category uses meta-agents to generate MASs.
Some approaches use prompted LLMs to construct and revise workflows, such as AutoAgents~\citep{AutoAgents_10.24963/ijcai.2024/3}, EvoMAS~\citep{EvoMASGeneration_hu2026evomas}, and ASPEC~\citep{ASPEC_ICLR2026_9230cf96}.
Others train MAS-generation policies through supervised fine-tuning~\citep{MAS-GPT_pmlr-v267-ye25g}, preference optimization~\citep{Scoreflow_wang2025scoreflow,MAS2_ICLR2026_b8831817}, or reinforcement learning (RL)~\citep{MASOrchestra_ke2026masorchestra,Conductor_ICLR2026_dbb8193a,Flowreasoner_gao2025flowreasoner}.

\textbf{Reward-driven policy learning.} In query-level MAS search, discrete architectural decisions require downstream execution to evaluate their utility, making direct end-to-end differentiation difficult.
Accordingly, most query-adaptive methods discussed above learn MAS policies from execution feedback through policy-gradient RL or preference optimization.
Specifically, these policies guide operator and agent selection~\citep{MaAS_pmlr-v267-zhang25bi,RAAS_yang2026raas,Puppter_NEURIPS2025_f1320d2e}, communication topology generation~\citep{G-Designer_pmlr-v267-zhang25cu,RADAR_zhang2026radar}, complete workflow generation~\citep{MASOrchestra_ke2026masorchestra,MAS2_ICLR2026_b8831817,Conductor_ICLR2026_dbb8193a}, or higher-level decisions about architecture reuse~\citep{ASPEC_ICLR2026_9230cf96}.
However, this widespread adoption raises a fundamental question: \textit{do these methods actually obtain informative execution feedback for learning a MAS policy?}
We examine this question from two perspectives: identifying queries that provide useful learning signals and distinguishing MAS substructures that contribute to successful execution.

\begin{figure}[t]
    \centering
    \includegraphics[width=\linewidth]{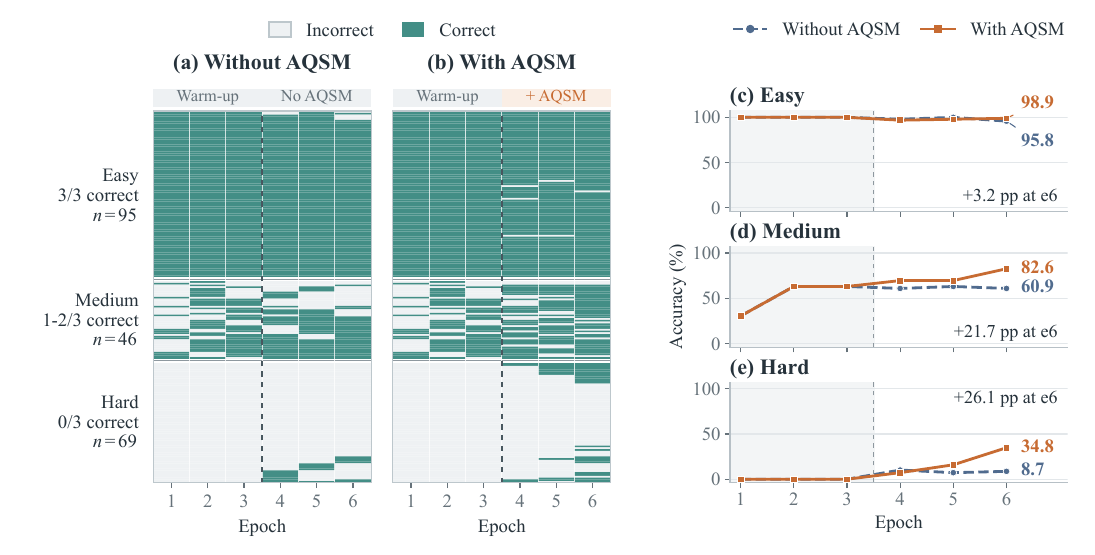}
    \caption{Query-level outcomes and group-wise accuracy on MATH, evaluated on the full training set after each training epoch, with and without AQSM.
    Queries are grouped by their outcomes without AQSM during warm-up (epochs 1--3): easy (3/3 correct), medium (1--2/3), and hard (0/3).}
    \label{fig:query-learning-al}
\vspace{-1em}
\end{figure}

\textbf{Question \#1. \textit{Uneven query informativeness:}} \textit{Does policy learning prioritize queries that provide informative execution feedback?}
Existing reward-driven approaches typically revisit a fixed set of training queries over multiple epochs, without adapting query selection to their evolving learning value.
To examine query-level learning dynamics, we train our RL-based MAS search method with GPT-4o-mini on 210 MATH training queries for six epochs.
Figure~\ref{fig:query-learning-al}(a) shows the results under the strategy of uniformly revisiting all queries.
Across all six epochs, 90 queries (42.9\%) are consistently answered correctly, 66 (31.4\%) exhibit both correct and incorrect outcomes, and 54 (25.7\%) are never answered correctly.
These trajectories suggest that queries differ in their remaining learning potential.
Consistently solved queries offer limited room for improvement, while persistently unsuccessful ones may exceed the capabilities of the current model and operator pool.
Uniformly revisiting both groups can therefore consume training budget while providing little correctness-based feedback to guide architectural decisions.
\textbf{Conclusion \#1: These observations motivate adaptive training query selection: identifying queries that remain learnable and prioritizing those whose execution feedback can most effectively improve the policy.}

\textbf{Question \#2. \textit{Coarse-grained substructure attribution:}} \textit{Does execution feedback reveal which MAS substructures improve answer quality?}
Existing reward-driven approaches commonly optimize architectural policies using workflow-level feedback, which may reinforce unnecessary steps in successful workflows or overlook useful steps in failed ones~\citep{MaAS_pmlr-v267-zhang25bi,Conductor_ICLR2026_dbb8193a,Scoreflow_wang2025scoreflow,MAS2_ICLR2026_b8831817}, thereby limiting the policy's ability to learn effective and efficient multi-agent systems.
We compare our method with MaAS~\citep{MaAS_pmlr-v267-zhang25bi} on MATH after training, scoring each layer's answer as $1$ if correct and otherwise by its BERT-based similarity to the reference answer.
For our method, average answer quality increases by $+0.033$, $+0.030$, $+0.024$, and $+0.018$ across successive transitions from L1 to L5.
By contrast, MaAS yields gains of $+0.026$, $+0.005$, and $+0.012$ over the first three transitions, but answer quality drops by $0.011$ from L4 to L5.
\textbf{Conclusion \#2: This motivates fine-grained credit assignment to reinforce MAS decisions that improve answer quality and discourage redundant or harmful computation.}

Building on these two conclusions, we propose Active Substructure-aware Policy Optimization (ASPO), a framework for query-level MAS search.
To address Question~\#1, we introduce an Adaptive Query-Selection Mechanism (AQSM) that prioritizes queries with remaining learning potential.
Specifically, after warm-up epochs, we maintain a success-probability posterior for each query and prioritize queries at the current policy's competence boundary, namely those it can solve but not yet reliably.
This reduces training effort spent on queries that are already consistently solved or unsolvable with the available agents.
For queries not yet solved, a discovery mechanism explores additional rollouts and retains useful trajectories to avoid mistaking insufficient exploration for inability.
To evaluate AQSM, we group queries in Figure~\ref{fig:query-learning-al}(a, b) by their outcomes without AQSM over the three warm-up epochs: easy ($n=95$), medium ($n=46$), and hard ($n=69$).
After identical warm-up procedures, we compare per-group accuracy from epoch 4 onward, when AQSM is enabled in one run.
As shown in Figure~\ref{fig:query-learning-al}(c--e), overall accuracy without AQSM barely improves after warm-up, rising from 59.0\% to 59.5\%.
By contrast, AQSM raises overall accuracy to 74.3\%.
Its largest gains occur on medium and hard queries, where accuracy reaches 82.6\% and 34.8\%, compared with 60.9\% and 8.7\% without AQSM, while maintaining high accuracy on easy queries.
To address Question~\#2, we define a substructure-level reward based on improvements in answer quality from parent agents to their children.
We then optimize the MAS policy using Proximal Policy Optimization (PPO) with these localized rewards, reinforcing decisions that produce useful downstream refinements.
Appendix~\ref{app:related-work} places these two components in the context of related work on adaptive data selection and credit assignment in other domains.
Our contributions are as follows:
\begin{itemize}[leftmargin=1.2em,labelsep=0.4em,topsep=0pt,partopsep=0pt,itemsep=2pt,parsep=0pt]
    \item We identify two limitations of execution feedback in query-level MAS policy learning: uneven query informativeness and coarse-grained substructure attribution.
    \item We propose an active, substructure-aware RL framework that combines AQSM for adaptive query selection with substructure-level rewards for policy optimization.
    \item We evaluate ASPO on six benchmarks spanning mathematical reasoning, general question answering, and code generation, demonstrating its consistent superiority over twelve baselines.
\end{itemize}

%% file: tex/problem_formulation.tex
\section{Problem Formulation}
\label{sec:problem-formulation}

\textbf{Multi-agent systems.}
Following MaAS~\citep{MaAS_pmlr-v267-zhang25bi}, we define the search space through a pool of agentic operators $\mathcal{O}$.
Each operator specifies an LLM-based operation for generating or refining a response.
A MAS is represented as a directed acyclic graph $\mathcal{G}=(\mathcal{V},\mathcal{E})$, where each node instantiates an operator from $\mathcal{O}$ and each edge specifies the flow of outputs between nodes.
Given a training set $\mathcal{D}_{\mathrm{train}}=\{(q_n,y_n)\}_{n=1}^{N}$ of queries and reference answers, executing a graph on query $q_n$ produces intermediate responses and a final answer $\hat{y}_n$.

\textbf{Query-level MAS search.}
We aim to learn an architectural policy $\pi_\theta$ that sequentially selects operators and their connections conditioned on the query and the current partial graph.
This policy induces a distribution $\pi_\theta(\mathcal{G}\mid q)$ over valid graphs in the search space.
Let $U(\mathcal{G};q)$ denote the task utility obtained by executing $\mathcal{G}$ on query $q$.
We optimize the policy to maximize expected utility:
\begin{equation}
    \max_{\theta}\;
    \mathbb{E}_{(q,y)\sim\mathcal{D}_{\mathrm{train}}}
    \mathbb{E}_{\mathcal{G}\sim\pi_\theta(\cdot\mid q)}
    \left[U(\mathcal{G};q)\right].
    \label{eq:search-objective}
\end{equation}
With a limited budget, we study how to extract more informative training feedback to learn $\pi_\theta$.

%% file: tex/method.tex
\section{Method}
\label{sec:method}

\begin{figure}[t]
    \centering
    \includegraphics[width=\linewidth]{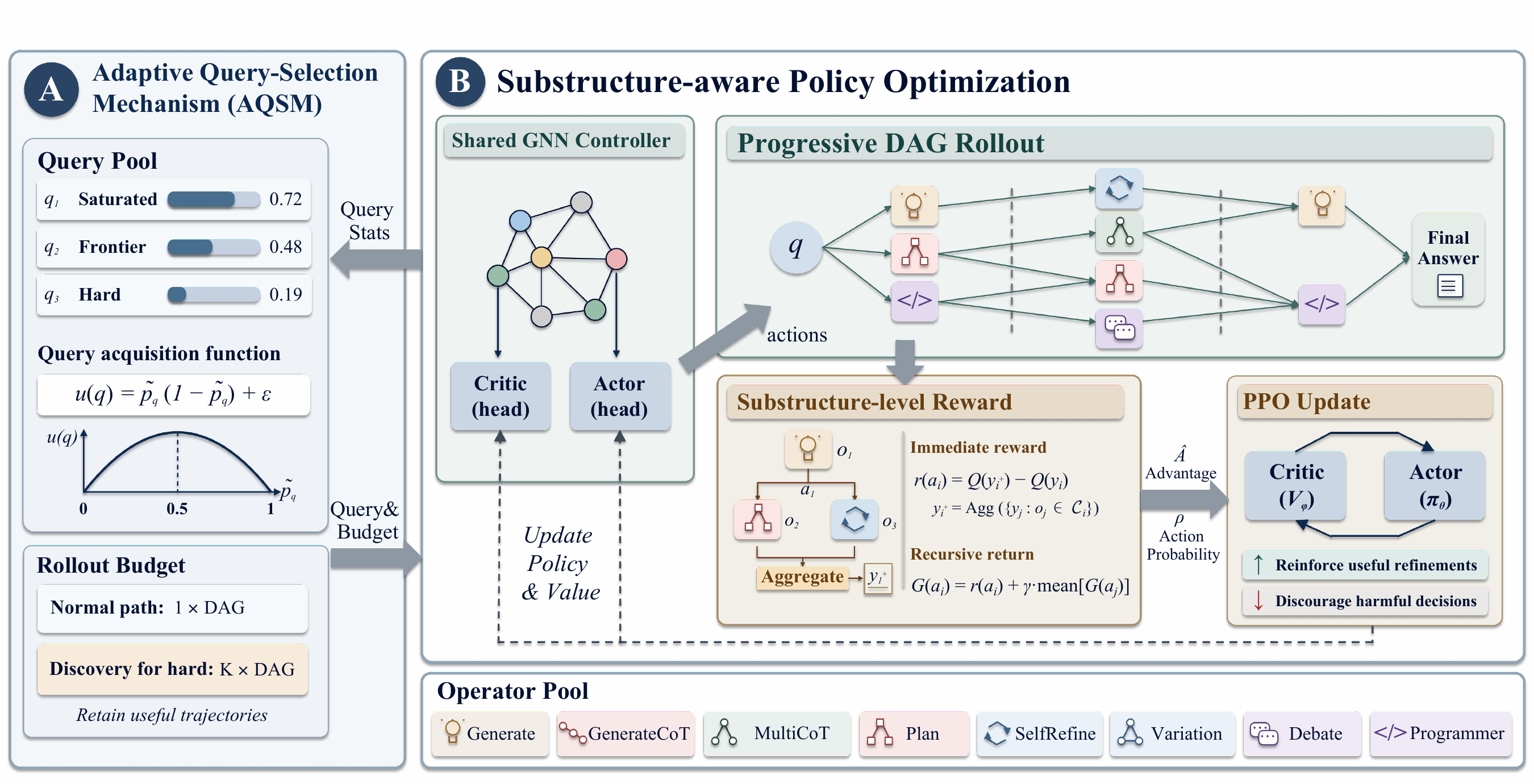}
    \caption{Overview of ASPO.
    (A) AQSM uses execution feedback to select queries and allocate rollouts.
    (B) The policy constructs query-specific MASs and learns from substructure-level rewards.}
    \label{fig:aspo-overview}
\vspace{-1em}
\end{figure}

In this section, we introduce ASPO, a framework for query-level MAS search through \underline{\textbf{A}}ctive \underline{\textbf{S}}ubstructure-aware
  \underline{\textbf{P}}olicy \underline{\textbf{O}}ptimization.
  As illustrated in Figure~\ref{fig:aspo-overview}, ASPO comprises two main modules:
  (i) adaptive query selection (Section~\ref{sec:aqsm}), which prioritizes training queries based on their evolving learning potential;
  and
  (ii) substructure-aware policy optimization (Section~\ref{sec:substructure-policy-optimization}), which uses quality improvements within each action's
  descendant subgraph to compute localized returns for PPO updates.
  For each selected query, the architectural policy progressively constructs and executes a MAS (Section~\ref{sec:mas-construction}), providing execution
  feedback for both modules.

\subsection{Policy-guided MAS Construction}
\label{sec:mas-construction}
We formulate MAS construction as a sequence of graph-expansion actions, starting from a root node representing the input query.
A GCN-based state encoder captures the query and the current partial graph to guide these actions.
Conditioned on this representation, the policy determines whether and how to expand each parent node using operators from $\mathcal{O}$.

\textbf{State representation.}
To represent the query and the current partial graph, we use a state encoder based on a graph convolutional network (GCN)~\citep{GNN_kipf2017semisupervised}.
At expansion action $i$, let $\mathcal{G}_i$ denote the partial graph and $v_i$ the parent node to be expanded.
We first project each node's operator embedding into the hidden space, then apply a GCN along parent-to-child edges with self-loops to obtain node representations.
These representations are combined through attention pooling to form the graph representation $\mathbf{h}_{\mathcal{G}_i}$.
We then concatenate $\mathbf{h}_{\mathcal{G}_i}$ with the parent representation $\mathbf{h}_{v_i}$ and query encoding $\mathbf{h}_q$ to form the state:
$
    \mathbf{s}_i =
    [\mathbf{h}_{\mathcal{G}_i}
    \,\Vert\, \mathbf{h}_{v_i}
    \,\Vert\, \mathbf{h}_q],
$
where $\mathbf{h}_q$ is obtained using a lightweight sentence encoder such as MiniLM~\citep{MiniLM_NEURIPS2020_3f5ee243} or Sentence-BERT~\citep{Sentence-Bert_reimers-gurevych-2019-sentence}.
The resulting state representation serves as a shared input to the actor and critic used in the policy optimization below.
Implementation details are provided in Appendix~\ref{app:state-encoder}.

\textbf{Progressive DAG rollout.}
Given the state representation $\mathbf{s}_i$, the policy expands the current parent node $v_i$ through the progressive DAG rollout shown in Figure~\ref{fig:aspo-overview}.
The actor implements this expansion through two action heads: a branching head and an operator-selection head.
The branching head samples the number of children $k_i$ from $\pi_\theta^{\mathrm{branch}}(k_i\mid\mathbf{s}_i)$, where $k_i=0$ stops expansion of the current parent.
When $k_i>0$, the operator-selection head sequentially samples $k_i$ compatible operators without replacement from $\pi_\theta^{\mathrm{op}}(o_{i,j}\mid\mathbf{s}_i,o_{i,<j})$, where $o_{i,<j}$ denotes the preceding selections.
We denote the action for expanding $v_i$ by $a_i=(k_i,o_{i,1},\ldots,o_{i,k_i})$, whose probability factorizes as
\begin{equation}
    \pi_\theta(a_i\mid\mathbf{s}_i)
    =
    \pi_\theta^{\mathrm{branch}}(k_i\mid\mathbf{s}_i)
    \prod\nolimits_{j=1}^{k_i}
    \pi_\theta^{\mathrm{op}}
    (o_{i,j}\mid\mathbf{s}_i,o_{i,<j}),
    \label{eq:architecture-action}
\end{equation}
where $o_{i,j}\in\mathcal{O}$ denotes the operator selected for the $j$-th child.
Each operator executes on the query and the outputs of its parent nodes.
When parents in the same layer select the same operator, they share a single child node that receives all their outputs.
Construction and execution proceed layer by layer until no operator node remains available for expansion or the construction budget is exhausted, such as reaching the maximum depth.
At that point, the outputs of all leaf operators are aggregated, for example through ensembling or voting, to produce the final answer.

\subsection{Substructure-aware Policy Optimization}
\label{sec:substructure-policy-optimization}
Building on the policy-guided MAS construction described above, we optimize the architectural policy using substructure-level execution feedback, as shown in Figure~\ref{fig:aspo-overview}(B).
For each action, we measure output-quality improvements within its descendant subgraph and aggregate them into an action-specific return.
These returns guide PPO updates.

\textbf{Substructure-level rewards.}
We assign rewards based on the quality improvement produced by each graph-expansion action.
For node output $y_i$ and reference solution $y^*$, we define
\begin{equation}
    Q(y_i) = \operatorname{sim}(y_i,y^*) + U_{\mathrm{task}}(y_i),
    \qquad U_{\mathrm{task}}(y_i)\in[0,1],
    \label{eq:output-quality}
\end{equation}
where $U_{\mathrm{task}}(y_i)\in[0,1]$ measures task completion using the benchmark evaluator, such as answer correctness or the fraction of test cases passed.
However, this task-specific utility alone can provide sparse learning signals.
To provide a finer-grained auxiliary signal, we introduce $\operatorname{sim}(y_i,y^*)\in[0,1]$, which measures alignment with the reference solution using a task-dependent similarity function (detailed in Appendix~\ref{app:semantic-similarity}).
For action $a_i$ expanding node $v_i$, we then aggregate the outputs of its children $\mathcal{C}_i$ and measure their collective improvement over the parent:
\begin{equation}
    y_i^{+} = \operatorname{Agg}\bigl(\{y_j : j \in \mathcal{C}_i\}\bigr),
    \qquad
    r(a_i) = Q(y_i^{+}) - Q(y_i).
    \label{eq:substructure-reward}
\end{equation}
This local aggregation evaluates the collective outcome of an expansion rather than scoring each child independently.
Finally, to account for downstream refinements and execution cost, we define the substructure-level rewards $G(a_i)$ over the descendant subgraph:
\begin{equation}
    G(a_i)
    =
    r(a_i)
    +
    \frac{\gamma}{|\mathcal{C}_i|}
    \sum_{j \in \mathcal{C}_i} G(a_j)
    - \lambda\,\Delta C_i,
    \label{eq:substructure-return}
\end{equation}
where $a_j$ denotes the expansion action at child node $v_j$, $\gamma\in[0,1]$ discounts more distant returns, and $\lambda\geq0$ weights the incremental execution cost $\Delta C_i$,
measured in FLOPs or tokens~\citep{pmlr-v235-sardana24a}.
For stop actions ($k_i=0$), we set $r(a_i)=G(a_i)=0$.
Each action thus receives a return determined by its own descendant subgraph rather than a shared workflow-level reward.

\textbf{PPO updates.}
Using the substructure-level rewards, we now optimize the architectural policy $\pi_\theta$ as the actor in PPO~\citep{PPO_schulman2017proximal}.
To estimate the advantage of each expansion action, a critic $V_\phi(\mathbf{s}_i)$ first predicts the expected return from expanding the current parent node.
For rollouts collected under $\pi_{\theta_{\mathrm{old}}}$, this gives the following advantage estimate and action probability ratio:
\begin{equation}
    \hat{A}_i = G(a_i)-V_{\phi_{\mathrm{old}}}(\mathbf{s}_i),
    \qquad
    \rho_i(\theta)=
    \frac{\pi_\theta(a_i\mid\mathbf{s}_i)}
    {\pi_{\theta_{\mathrm{old}}}(a_i\mid\mathbf{s}_i)},
    \label{eq:ppo-advantage-ratio}
\end{equation}
where the action probability follows Equation~\ref{eq:architecture-action} and $\phi_{\mathrm{old}}$ denotes the critic parameters used during rollout collection.
The actor maximizes the clipped surrogate objective
\begin{equation}
    \mathcal{L}_{\mathrm{policy}}(\theta)
    = \widehat{\mathbb{E}}_i\!\left[
        \min\!\left(
            \rho_i(\theta)\hat{A}_i,
            \operatorname{clip}\!\left(\rho_i(\theta),1-\epsilon_{\mathrm{clip}},1+\epsilon_{\mathrm{clip}}\right)\hat{A}_i
        \right)
    \right],
    \label{eq:ppo-policy-objective}
\end{equation}
where $\widehat{\mathbb{E}}_i$ denotes the empirical average over collected expansion actions and $\epsilon_{\mathrm{clip}}$ controls clipping.
This objective encourages positive-advantage actions and suppresses negative-advantage actions.
On the other hand, the critic is trained to predict the expected substructure returns by minimizing
\begin{equation}
    \mathcal{L}_{\mathrm{value}}(\phi)
    = \widehat{\mathbb{E}}_i\!\left[
        \left(V_\phi(\mathbf{s}_i)-G(a_i)\right)^2
    \right].
    \label{eq:ppo-value-loss}
\end{equation}
The actor and critic share the state encoder, which receives gradients from both objectives.

\subsection{Adaptive Query-Selection Mechanism}
\label{sec:aqsm}
While substructure-level rewards provide localized feedback within each rollout, their usefulness also depends on the queries selected for rollouts.
Instead of uniformly revisiting training queries as existing methods do, our Adaptive Query-Selection Mechanism (AQSM), shown in
Figure~\ref{fig:aspo-overview}(A), prioritizes queries based on their evolving learning potential and activates discovery for hard queries.

\textbf{Competence-boundary query selection.}
We seek to allocate more training effort to queries at the current policy's competence boundary, rather than those it consistently solves or cannot yet solve with the available operators.
Our insight is that queries the policy can solve but not yet reliably retain room for improvement and may provide informative execution feedback for policy learning.
We therefore estimate each query's success probability under $\pi_\theta$ as a proxy for its current learning potential.
For each query $q$, we maintain discounted success and failure counts, $S_q$ and $F_q$.
Starting from a uniform prior, we represent its success probability $p_q$ using a Beta posterior:
\begin{equation}
    p_q \sim \operatorname{Beta}(1+S_q,\,1+F_q).
    \label{eq:query-success-posterior}
\end{equation}
After each standard rollout, we update the counts using its binary success indicator $b_q\in\{0,1\}$:
\begin{equation}
    S_q \leftarrow \delta S_q+b_q,
    \qquad
    F_q \leftarrow \delta F_q+(1-b_q),
    \label{eq:query-posterior-update}
\end{equation}
where $\delta\in(0,1]$ discounts earlier observations so that the estimate can track changes in policy capability.
To prioritize queries at the competence boundary, we use binary outcome variance as a proxy for learning potential.
For each query, we sample $\tilde{p}_q$ from its current posterior and compute:
\begin{equation}
    w(q)=\tilde{p}_q(1-\tilde{p}_q),
    \label{eq:query-selection-weight}
\end{equation}
which peaks at $\tilde{p}_q=0.5$, favoring queries that are solvable but not yet reliably solved.
Appendix~\ref{app:query-selection-theory} motivates this choice through an improvement bound proportional to $p_q(1-p_q)$ for a gradient step on query success probability, and explains how the discounted posterior sustains exploration throughout training.
After $T_{\mathrm{w}}$ warm-up epochs, we activate \textbf{adaptive query selection} by redrawing $\tilde{p}_q$ before each batch and sampling queries in proportion to the resulting scores.

\textbf{Discovery for difficult queries.}
However, queries downweighted for remaining consistently unsolved may fall into two cases: those beyond the current capabilities of the available models and operators, and those for which effective architectures have not yet been explored.
We therefore introduce a \textbf{discovery mode} to distinguish insufficient exploration from capability limits.
Specifically, after the warm-up epochs, a sampled query enters discovery mode if none of its previous rollouts has yielded a correct answer.
For each query in discovery mode, we collect $K$ rollouts under $\pi_{\theta_{\mathrm{old}}}$, subject to the remaining per-epoch budget, compared with one for each other sampled query.
To focus PPO updates on informative outcomes, we retain a discovery trajectory only if its final answer is correct or at least one architectural action has a positive substructure advantage, $\hat{A}_i>0$.
Discovery thus helps distinguish underexplored queries from those beyond current capabilities, helping them move toward the competence boundary and increase their future sampling weights.

\subsection{Overall Algorithm}
\label{sec:overall-algorithm}
Algorithm~\ref{alg:aspo} in Appendix~\ref{app:aspo-algorithm} summarizes the ASPO training procedure.
After warm-up, AQSM selects queries for training and activates discovery for sampled queries that have never been solved.
For each query, the policy progressively constructs and executes MAS, using substructure rewards for PPO updates.
Once trained, the policy $\pi_\theta$ constructs a MAS tailored to each query at inference.

%% file: tex/experiment.tex
\section{Experiments}
\label{sec:experiments}

\subsection{Experimental Setup}
\label{sec:exp-setup}

\textbf{Benchmarks and metrics.}
We evaluate ASPO on six public benchmarks covering three domains:
(i) \textbf{mathematical reasoning}, GSM8K~\citep{GSM8K_cobbe2021training} and MATH~\citep{MATH_hendrycks2021measuring};
(ii) \textbf{general question answering}, HotpotQA~\citep{HotpotQA_yang2018hotpotqa} and MMLU-Pro~\citep{MMLUPro_wang2024mmlu};
and (iii) \textbf{code generation}, MBPP+ and HumanEval+~\citep{EvalPlus_liu2023your}.
We report accuracy for GSM8K, MATH and MMLU-Pro, F1 for HotpotQA, and pass@1 for MBPP+ and HumanEval+.
Data sources, selection procedures, and train/test splits are detailed in Appendix~\ref{app:benchmarks}.

\textbf{Baselines.}
We compare ASPO with twelve baselines spanning four categories:
(i) \textbf{single-agent methods}, Vanilla and CoT~\citep{CoT_wei2022chain};
(ii) \textbf{multi-agent collaboration}, AgentVerse~\citep{AgentVerse_chen2024agentverse} and DyLAN~\citep{DyLAN_liu2024a};
(iii) \textbf{task-level MAS search}, ADAS~\citep{ADAS_hu2025automated}, AFlow~\citep{AFlow_zhang2025aflow} and EvoMAS~\citep{EvoMASHiL_wei2026evomas};
and (iv) \textbf{query-level MAS search}, which we group following Section~\ref{sec:intro}: \emph{agent selection}, MaAS~\citep{MaAS_pmlr-v267-zhang25bi}, RAAS~\citep{RAAS_yang2026raas} and Puppeteer~\citep{Puppter_NEURIPS2025_f1320d2e} (ASPO's category and its most direct baselines); \emph{communication topology}, RADAR~\citep{RADAR_zhang2026radar}; and \emph{meta-agents}, ASPEC~\citep{ASPEC_ICLR2026_9230cf96}.
Each baseline is run with its official implementation, as detailed in Appendix~\ref{app:baseline-setups}.

\textbf{Implementation details.}
The search space of ASPO comprises ten agents described in Appendix~\ref{app:operator-space}, which is also shared with the three agent-selection baselines.
We use \texttt{gpt-4o-mini-2024-07-18}~\citep{gpt-4o-mini} as the LLM backbone for all methods, with temperature $1$ for ASPO.
We set the maximum number of layers to $5$ and the maximum number of children per node to $2$.
For the substructure reward, we set the cost coefficient as $\lambda=0.01$.
For AQSM, we discount query statistics with $\delta=0.9$ and use a discovery budget of $K=3$ rollouts per query.
We train for $T=6$ epochs with $T_{\mathrm{w}}=3$ warm-up epochs, and cap the number of rollouts per epoch at the size of the training set.
Appendix~\ref{app:hyperparameters} details PPO training and aggregation rules, and provides a summary table linking hyperparameter settings to their sensitivity analyses.

\subsection{Main Results}
\label{sec:exp-main}
\input{tex/main_table}

Table~\ref{tab:main} reports performance on the six benchmarks for ASPO and the twelve baselines.
From these results we draw three conclusions.
\textbf{(i) ASPO ranks first on every benchmark, while no baseline holds a stable position.}
Its average of $78.19$ exceeds the $74.06$ of the strongest baseline, DyLAN, whose average rank is $3.50$.
The baselines are far less consistent: MaAS ranks second on MATH but eleventh on MMLU-Pro, Puppeteer third on MBPP+ but thirteenth on HumanEval+, and the runner-up position rotates among DyLAN, RADAR and MaAS.
\textbf{(ii) The gain holds in every task domain.}
Averaged within each domain, ASPO leads the strongest baseline for that domain by $4.1$ points on mathematical reasoning, $3.5$ on code generation and $3.0$ on general question answering, and the identity of that baseline changes from domain to domain.
\textbf{(iii) With the same agent pool and the same learning paradigm, the advantage comes from how the policy is trained.}
MaAS, RAAS and Puppeteer also learn agent selection from execution feedback with RL, yet ASPO exceeds the best of the three by $2.3$ to $5.8$ points on every benchmark.
We attribute this to the two components that distinguish our framework: AQSM focuses rollouts on queries with remaining learning potential, and substructure-level rewards credit individual expansion actions rather than scoring the MAS as a whole.
Section~\ref{sec:ablation} isolates the contribution of each component.

\begin{figure}[t]
    \centering
    \begin{subfigure}[b]{0.57\linewidth}
        \centering
        \includegraphics[width=\linewidth]{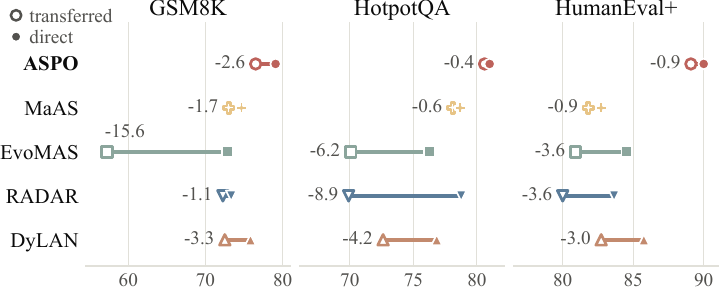}
        \caption{Task transfer: gap between transfer and direct training.}
        \label{fig:domain-transfer}
    \end{subfigure}\hfill
    \begin{subfigure}[b]{0.415\linewidth}
        \centering
        \includegraphics[width=\linewidth]{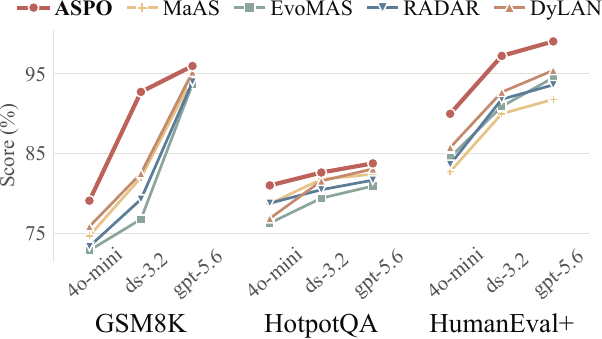}
        \caption{LLM transfer: score per backbone.}
        \label{fig:llm-transfer}
    \end{subfigure}
    \caption{Transferability of the learned policy, for the top-5 methods on representative benchmarks.}
    \label{fig:transfer}
\end{figure}

\subsection{Transferability}
\label{sec:exp-transfer}
In this section, we compare the top five methods in Table~\ref{tab:main} to evaluate whether they transfer to task domains and LLM backbones that differ from those used during training.
To limit experimental cost, this and subsequent analyses use three representative benchmarks: GSM8K, HotpotQA, and HumanEval+, selected by their agreement with the overall method ranking (Appendix~\ref{app:representative-benchmarks}).
(i) \textbf{Transfer across task domains.}
We train each method on MATH and evaluate it on the three representative benchmarks without further adaptation, using \texttt{gpt-4o-mini} throughout.
Figure~\ref{fig:domain-transfer} compares these results with training directly on each target benchmark.
ASPO retains the highest score on all three benchmarks, while ASPO and MaAS show smaller average performance drops than the other methods.
By contrast, EvoMAS transfers a fixed MAS optimized on MATH and loses $8.5$ points on average.
DyLAN and RADAR retain the agent team or role pool fixed during training, limiting their ability to adapt to the target task's role requirements.
Their performance therefore drops to varying degrees depending on how well the retained roles suit the target task, with DyLAN losing $3.5$ points on average and RADAR losing $8.9$ points on HotpotQA.
(ii) \textbf{Transfer across LLM backbones.}
Here, training and inference use the same benchmark but different LLM backbones: each method is trained with \texttt{gpt-4o-mini} and evaluated with \texttt{deepseek-v3.2}~\citep{deepseek-v3.2} or \texttt{gpt-5.6-terra}~\citep{gpt-5.6-terra}.
As shown in Figure~\ref{fig:llm-transfer}, the ordering of the five methods remains largely consistent across different LLM backbones, with ASPO ranking first in all nine benchmark--backbone combinations.
This suggests that the comparative advantage of the learned policy persists across execution LLMs.

\begin{figure}[t!]
    \centering
    \begin{minipage}[b]{0.44\linewidth}
        \centering
        \includegraphics[width=\linewidth]{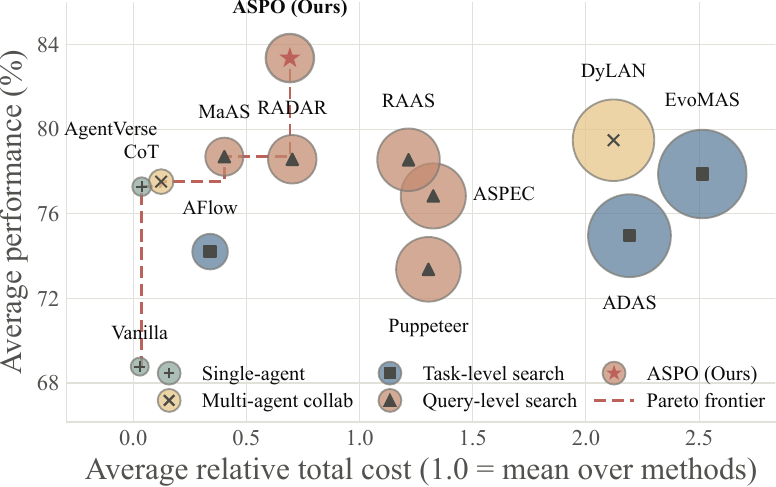}
        \caption{Cost-performance trade-off.}
        \label{fig:cost-bubble}
    \end{minipage}\hfill
    \begin{minipage}[b]{0.55\linewidth}
        \centering
        \footnotesize
        \setlength{\tabcolsep}{2.5pt}
        \scalebox{0.85}{%
        \begin{tabular}{lcccc}
        \toprule
          & \textbf{GSM8K} & \textbf{HotpotQA} & \textbf{HumanEval+} & \textbf{Avg.} \\
        \midrule
        \textbf{ASPO} & \textbf{79.08} & \textbf{81.00} & \textbf{90.00} & \textbf{83.36} \\
        \midrule
        \addlinespace[4pt]
        \rowcolor{gray!15}
        \multicolumn{5}{c}{\textit{Adaptive query selection (Section~\ref{sec:aqsm})}} \\
        \addlinespace[3pt]
        \textit{w/o} discovery   & 76.75\textsubscript{$-$2.33} & 79.28\textsubscript{$-$1.72} & 89.09\textsubscript{$-$0.91} & 81.71\textsubscript{$-$1.65} \\
        \textit{w/o} posterior   & 75.75\textsubscript{$-$3.33} & 79.04\textsubscript{$-$1.96} & 87.27\textsubscript{$-$2.73} & 80.69\textsubscript{$-$2.67} \\
        \textit{w/o} AQSM        & 75.25\textsubscript{$-$3.83} & 78.57\textsubscript{$-$2.43} & 88.18\textsubscript{$-$1.82} & 80.67\textsubscript{$-$2.69} \\
        \addlinespace[3pt]
        \midrule
        \addlinespace[4pt]
        \rowcolor{gray!15}
        \multicolumn{5}{c}{\textit{Substructure-aware reward (Section~\ref{sec:substructure-policy-optimization})}} \\
        \addlinespace[3pt]
        \textit{w/o} similarity  & 76.50\textsubscript{$-$2.58} & 79.89\textsubscript{$-$1.11} & 86.36\textsubscript{$-$3.64} & 80.92\textsubscript{$-$2.44} \\
        \textit{w/o} sub-reward  & 75.50\textsubscript{$-$3.58} & 78.83\textsubscript{$-$2.17} & 86.36\textsubscript{$-$3.64} & 80.23\textsubscript{$-$3.13} \\
        \addlinespace[3pt]
        \bottomrule
        \end{tabular}}
        \captionof{table}{Ablation study: subscripts give the drop.}
        \label{tab:ablation}
    \end{minipage}
\vspace{-1em}
\end{figure}

\subsection{Cost Analysis}
\label{sec:exp-cost}

We next examine the computational cost of all methods.
Figure~\ref{fig:cost-bubble} compares average performance against total monetary cost on the three representative benchmarks, where costs are normalized by the mean across methods within each benchmark before averaging.
We draw two conclusions.
\textbf{(i) ASPO occupies the highest-performance end of the Pareto frontier.}
ASPO achieves the highest average performance of $83.36$ at a normalized cost of $0.69$.
MaAS costs less ($0.40$) but scores $4.66$ points lower, while RADAR incurs a similar cost ($0.70$) yet scores $4.78$ points lower.
\textbf{(ii) ASPO achieves higher performance without relying on greater expenditure.}
Every method more expensive than ASPO performs worse.
In particular, compared with DyLAN, the strongest baseline, ASPO improves performance by $3.88$ points while reducing cost by $67.4\%$.
Appendix~\ref{app:cost-analysis} further reports per-benchmark training and inference costs and compares wall-clock times.

\subsection{Ablation study}
\label{sec:ablation}

In this section, we isolate the contribution of each component of ASPO by removing one at a time.
Table~\ref{tab:ablation} reports the results, where every variant loses accuracy on every benchmark.
\textbf{(i) Ablation of adaptive query selection.}
Removing discovery costs $1.65$ points, highlighting the value of additional exploration for unsolved queries.
Disabling AQSM entirely restores uniform query selection and increases the drop to $2.69$ points.
Moreover, the mechanism helps only if $p_q$ is estimated well: replacing it with the raw empirical success rate costs $2.67$, almost as much.
\textbf{(ii) Ablation of substructure-aware policy optimization.}
The \textit{w/o} sub-reward variant replaces substructure-level rewards with an end-to-end reward (as in MaAS) and costs $3.13$ points on average, since no individual expansion receives credit of its own.
Even with this attribution retained, removing semantic similarity causes a $2.44$-point drop.
Because the similarity term provides finer-grained feedback on intermediate refinements that may not change the task evaluator's score.
\textbf{(iii) Component contributions vary across benchmarks.}
Removing AQSM hurts GSM8K and HotpotQA more than HumanEval+, suggesting that query selection benefits from variation in query difficulty.
Conversely, removing substructure-level rewards hurts GSM8K and HumanEval+ more than HotpotQA, possibly because derivations and programs offer richer signals of partial progress than short answers.

\subsection{Sensitivity analysis}
\label{sec:sensitivity}

\begin{figure}[t!]
    \centering
    \begin{subfigure}[b]{0.66\linewidth}
        \centering
        \includegraphics[width=\linewidth]{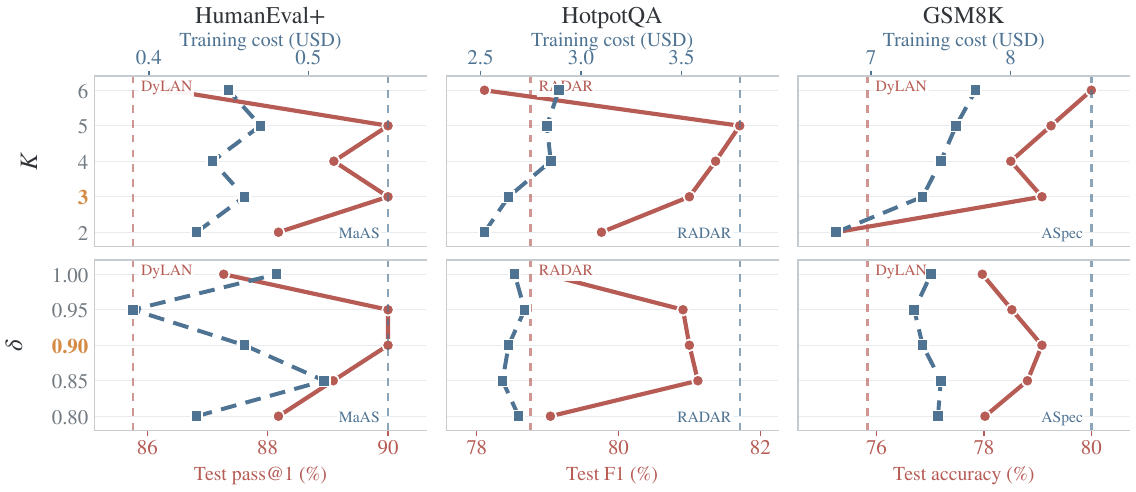}
        \caption{Posterior discount $\delta$ and discovery budget $K$.}
        \label{fig:sensitivity-delta}
        \label{fig:sensitivity-K}
    \end{subfigure}\hfill
    \begin{subfigure}[b]{0.32\linewidth}
        \centering
        \includegraphics[width=\linewidth]{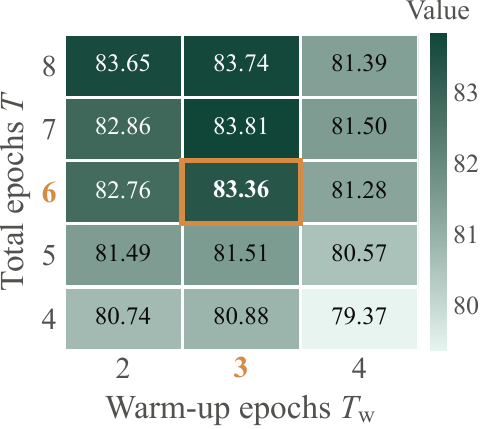}
        \caption{Epoch settings.}
        \label{fig:sensitivity-epochs}
    \end{subfigure}
    \caption{Sensitivity analysis of ASPO hyperparameters over three representative benchmarks.}
    \label{fig:hyperparameter-sensitivity}
\end{figure}

In this section, we examine ASPO's hyperparameter sensitivity for its two main components, following the grouping in Table~\ref{tab:hyperparameters}.
\textbf{(i) Adaptive query selection.}
Figure~\ref{fig:sensitivity-delta} examines $\delta$ and $K$, with red and blue curves showing test performance and training cost, respectively.
The matching vertical references indicate the strongest baseline's performance and the baseline cost immediately above ASPO's default cost.
Performance remains relatively stable for $\delta\in[0.85,0.95]$, a range that includes our default $\delta=0.9$.
More aggressive forgetting ($\delta=0.8$) or retaining all past observations ($\delta=1$) reduces performance, suggesting the importance of balancing recent and historical feedback.
Performance remains relatively stable for $K\in[3,5]$, which includes our default $K=3$, while smaller or larger budgets can cause substantial drops on some benchmarks.
Across most settings, performance changes are smaller than ASPO's margin over the strongest baseline, preserving its advantage.
Exceptions occur at the discovery-budget extremes: $K=2$ on GSM8K and $K=6$ on HotpotQA.
\textbf{(ii) Substructure-aware optimization.}
Figure~\ref{fig:sensitivity-epochs} reports average performance for different training and warm-up epochs.
Performance largely plateaus beyond our default training length of $T=6$.
A moderate warm-up of $T_{\mathrm{w}}=3$ generally performs better than shorter or longer warm-up periods.
We further examine the architecture-related hyperparameters, including $\lambda$ and the maximum numbers of layers and children per node, in
Appendix~\ref{app:architecture-sensitivity}.

%% file: tex/main_table.tex
\begin{table}[t!]
\centering
\footnotesize
\setlength{\tabcolsep}{2.5pt}
\scalebox{0.95}{%
\begin{tabular}{lcccccccc}
\toprule
\textbf{Method} & \textbf{GSM8K} & \textbf{MATH} & \textbf{HotpotQA} & \textbf{MMLU-Pro} & \textbf{MBPP+} & \textbf{HumanEval+} & \textbf{Avg.} & \textbf{AVG-R} \\
\midrule
\rowcolor{gray!15}
\multicolumn{9}{c}{\textit{Single-agent methods}} \\
Vanilla & 48.58\textsubscript{$\pm$1.70} & 47.48\textsubscript{$\pm$2.51} & 76.53\textsubscript{$\pm$2.69} & 61.83\textsubscript{$\pm$1.87} & 75.40\textsubscript{$\pm$2.34} & 81.21\textsubscript{$\pm$2.27} & 65.17 & 10.67 \\
CoT & 71.58\textsubscript{$\pm$1.76} & 48.44\textsubscript{$\pm$1.27} & 76.32\textsubscript{$\pm$1.68} & 73.17\textsubscript{$\pm$1.81} & 76.32\textsubscript{$\pm$2.11} & 83.94\textsubscript{$\pm$1.13} & 71.63 & 8.17 \\
\midrule
\rowcolor{gray!15}
\multicolumn{9}{c}{\textit{Multi-agent collaboration}} \\
AgentVerse & 73.08\textsubscript{$\pm$1.03} & 51.77\textsubscript{$\pm$1.00} & 77.36\textsubscript{$\pm$0.77} & 71.98\textsubscript{$\pm$1.76} & 74.60\textsubscript{$\pm$2.02} & 82.12\textsubscript{$\pm$1.13} & 71.82 & 8.00 \\
DyLAN & \underline{75.83\textsubscript{$\pm$1.96}} & 52.93\textsubscript{$\pm$1.67} & 76.84\textsubscript{$\pm$2.14} & \underline{75.48\textsubscript{$\pm$1.85}} & 77.51\textsubscript{$\pm$1.78} & \underline{85.76\textsubscript{$\pm$0.86}} & \underline{74.06} & \underline{3.50} \\
\midrule
\rowcolor{gray!15}
\multicolumn{9}{c}{\textit{Task-level MAS search}} \\
ADAS & 72.25\textsubscript{$\pm$2.55} & 40.00\textsubscript{$\pm$1.17} & 75.37\textsubscript{$\pm$2.50} & 71.90\textsubscript{$\pm$2.29} & 74.21\textsubscript{$\pm$3.09} & 77.27\textsubscript{$\pm$2.97} & 68.50 & 11.50 \\
AFlow & 70.58\textsubscript{$\pm$0.92} & 52.24\textsubscript{$\pm$1.59} & 67.50\textsubscript{$\pm$0.91} & 74.92\textsubscript{$\pm$1.56} & 76.06\textsubscript{$\pm$1.46} & 84.55\textsubscript{$\pm$0.74} & 70.98 & 7.50 \\
EvoMAS & 72.83\textsubscript{$\pm$1.33} & 49.46\textsubscript{$\pm$1.50} & 76.26\textsubscript{$\pm$0.98} & 74.29\textsubscript{$\pm$0.89} & 77.38\textsubscript{$\pm$2.12} & 84.55\textsubscript{$\pm$1.48} & 72.46 & 6.50 \\
\midrule
\rowcolor{gray!15}
\multicolumn{9}{c}{\textit{Query-level MAS search}} \\
MaAS & 74.67\textsubscript{$\pm$1.45} & \underline{53.67\textsubscript{$\pm$1.92}} & 78.69\textsubscript{$\pm$1.12} & 67.78\textsubscript{$\pm$1.75} & 75.53\textsubscript{$\pm$2.11} & 82.73\textsubscript{$\pm$0.74} & 72.18 & 6.00 \\
RAAS & 73.58\textsubscript{$\pm$1.45} & 53.61\textsubscript{$\pm$2.01} & 77.84\textsubscript{$\pm$1.18} & 69.21\textsubscript{$\pm$1.95} & 74.47\textsubscript{$\pm$1.46} & 84.24\textsubscript{$\pm$1.55} & 72.16 & 6.50 \\
Puppeteer & 73.42\textsubscript{$\pm$1.84} & 44.69\textsubscript{$\pm$2.00} & 77.31\textsubscript{$\pm$1.20} & 73.89\textsubscript{$\pm$0.98} & 78.44\textsubscript{$\pm$1.46} & 69.39\textsubscript{$\pm$1.87} & 69.52 & 7.67 \\
RADAR & 73.33\textsubscript{$\pm$1.84} & 47.28\textsubscript{$\pm$1.60} & \underline{78.76\textsubscript{$\pm$0.70}} & 73.97\textsubscript{$\pm$1.76} & \underline{82.80\textsubscript{$\pm$0.82}} & 83.64\textsubscript{$\pm$0.74} & 73.30 & 5.67 \\
ASPEC & 74.17\textsubscript{$\pm$1.96} & 52.38\textsubscript{$\pm$1.60} & 76.06\textsubscript{$\pm$1.94} & 67.22\textsubscript{$\pm$1.46} & 76.59\textsubscript{$\pm$1.80} & 80.30\textsubscript{$\pm$1.87} & 71.12 & 8.17 \\
\midrule
\textbf{ASPO (Ours)} & \textbf{79.08\textsubscript{$\pm$1.25}} & \textbf{57.89\textsubscript{$\pm$1.11}} & \textbf{81.00\textsubscript{$\pm$1.17}} & \textbf{77.70\textsubscript{$\pm$0.88}} & \textbf{83.47\textsubscript{$\pm$0.82}} & \textbf{90.00\textsubscript{$\pm$1.48}} & \textbf{78.19} & \textbf{1.00} \\
\bottomrule
\end{tabular}}
\caption{Main results with \texttt{gpt-4o-mini}, reported as mean\textsubscript{$\pm$std} over three runs. \textbf{Bold} marks the best and \underline{underline} the runner-up. AVG-R is the average rank (lower is better).}
\label{tab:main}
\vspace{-1em}
\end{table}

%% file: tex/conclusion.tex
\section{Conclusion}
\label{sec:conclusion}

In this paper, we presented ASPO, which combines adaptive query selection and discovery with substructure-level rewards to learn query-specific MAS policies.
ASPO outperforms twelve baselines across six benchmarks and retains its advantage across task domains and LLM backbones.
Ablations validate both components, while cost analyses show favorable performance--cost trade-offs.
ASPO currently searches within a predefined operator pool and uses reference solutions to compute training rewards.
Future work could explore applying ASPO to dynamically changing operator pools, and alternative training signals for tasks without reliable reference solutions (Appendix~\ref{app:limitations}).

%% file: appendix/related_work.tex
\section{Related Work}
\label{app:related-work}

The Introduction (Section~\ref{sec:intro}) reviews the background of MAS search and identifies two limitations of existing query-level methods that motivate ASPO: uniform allocation of training effort across queries and coarse workflow-level rewards.
Although these ideas remain underexplored in MAS search, they connect to broader research on adaptive data selection and credit assignment.
We therefore review work on adaptive training data selection and credit assignment, highlighting their connections to ASPO and the differences in their objectives and mechanisms.

\paragraph{Adaptive Training Data Selection.}
Early work on active learning sought to improve label efficiency by selecting informative examples for annotation~\citep{ALSurvey_settles2009active}.
Representative strategies include uncertainty sampling~\citep{UncertaintySampling_lewis1994sequential} and expected-gradient-length criteria that estimate an example's influence on model updates~\citep{EGL_settles2007multiple}.
Related ideas extend to reinforcement learning, where the budget is spent on task interactions and rollout execution.
Goal GAN generates goals within an intermediate success-rate range to build curricula that adapt to policy capability~\citep{GOID_florensa2018automatic}.
In LLM reinforcement learning, DAPO filters prompt groups with uniformly correct or incorrect outcomes~\citep{DAPO_yu2025dapo}, while GRESO uses historical feedback to skip uninformative prompts before rollout generation~\citep{REF_GRESO}.
MoPPS uses discounted Beta posteriors and posterior sampling to select prompts near a target success rate~\citep{REF_MoPPS}.
ASPO extends this line of work in two ways.
First, it brings adaptive training data selection to query-specific MAS architecture policy learning.
Second, while existing selection criteria often downweight or filter queries at both extremes of the success-rate spectrum, ASPO treats these extremes asymmetrically: consistently solved queries receive less training effort, whereas sampled queries that have never been solved receive additional rollouts through discovery.
This discovery mode helps distinguish queries that remain unsolved due to insufficient exploration from those beyond the capabilities of the current policy and operator pool.

\paragraph{Credit Assignment in Multi-Agent Systems.}
Shared global rewards make it difficult to distinguish individual agents' contributions, motivating more localized credit assignment.
COMA, for example, uses a counterfactual baseline to estimate an agent's action contribution while holding other agents' actions fixed~\citep{REF_COMA}.
In LLM reasoning, process supervision provides feedback on intermediate steps~\citep{REF_LetsVerifyStepByStep}.
Subsequent work uses Monte Carlo value estimation or implicit reward models trained on outcome labels to obtain fine-grained signals without manual step-level annotations~\citep{REF_VinePPO,REF_FreeProcessRewards}.
For interactive agents, GraphGPO assigns transition credit using a state-transition graph assembled from multiple rollouts~\citep{REF_GraphGPO}.
SHARP addresses credit assignment among LLM agents through Shapley-based marginal rewards~\citep{REF_SHARP}.
ASPO differs from these methods in both its optimization target and credit assignment mechanism.
First, these methods primarily assign credit to reasoning steps or agent actions within an execution framework, whereas ASPO uses localized feedback to optimize the MAS architecture policy, guiding operator selection, branching, and termination.
Second, rather than scoring individual reasoning steps or estimating agents' marginal contributions, ASPO assigns credit to each graph expansion by measuring the quality improvement of its aggregated child outputs over the parent output and incorporating discounted returns from its descendant subgraph together with execution costs.

%% file: appendix/limitations.tex
\section{Limitations and Future Work}
\label{app:limitations}

\textbf{Predefined operator pool.}
ASPO composes operators from a predefined pool, so its capabilities remain bounded by the available operators and underlying LLMs.
Discovery explores additional compositions but cannot introduce missing capabilities.
Future work could extend ASPO to dynamically changing operator pools and jointly learn operator design and composition.

\textbf{Dependence on reference-based training feedback.}
ASPO relies on reference solutions and task evaluators for training feedback, limiting its direct application to tasks without reliable references or success criteria.
Moreover, semantic similarity may also overlook valid alternative solutions or reward superficially similar but incorrect responses.
Future work could explore executable verification, human preferences, or learned evaluators as alternative signals for query selection and architectural credit assignment.

%% file: appendix/state_encoder.tex
\subsection{State Encoder Implementation}
\label{app:state-encoder}

\textbf{Node initialization.}
We use a hidden dimension of $d=128$ throughout the graph encoder.
Each operator, including the root, has an embedding initialized from its natural-language description using the frozen sentence encoder \texttt{all-MiniLM-L6-v2}~\citep{MiniLM_NEURIPS2020_3f5ee243}\footnote{\url{https://huggingface.co/sentence-transformers/all-MiniLM-L6-v2}}.
The resulting 384-dimensional embeddings are projected to $d$ dimensions to initialize an operator embedding table.
This initialization projection is applied only once; the table is initially frozen and subsequently optimized during policy training.
At expansion decision $i$, the input feature of node $v$ combines its operator embedding $\mathbf{e}_{o_v}$ with an indicator identifying the current parent $v_i$:
\begin{equation}
    \mathbf{x}_v
    =
    \mathbf{W}_{\mathrm{in}}
    [\mathbf{e}_{o_v}\Vert\mathbb{I}(v=v_i)]
    +\mathbf{b}_{\mathrm{in}},
    \qquad
    \mathbf{x}_v\in\mathbb{R}^{d}.
\end{equation}
Nodes instantiating the same operator share its embedding, while their structural contexts are captured by graph convolution.

\textbf{Graph convolution.}
We encode the current partial graph using two GCN layers~\citep{GNN_kipf2017semisupervised}, implemented with normalized \texttt{GCNConv} layers in PyTorch Geometric.
Edges retain the parent-to-child direction, and every node receives a self-loop.
Both layers preserve the hidden dimension $d$.
We apply ReLU after the first layer, omit the activation after the second layer, and use no dropout.
The resulting representation $\mathbf{h}_v$ incorporates information from the node itself and upstream nodes reachable within two message-passing steps.

\textbf{Graph-level pooling.}
A single-head attention pooling module summarizes the node representations into $\mathbf{h}_{\mathcal{G}_i}$.
Let $\mathcal{V}_i$ denote the nodes of the current partial graph and $\mathbf{u}\in\mathbb{R}^{d}$ a trainable pooling vector.
We compute
\begin{equation}
    \alpha_v
    =
    \frac{\exp(\mathbf{u}^{\top}\mathbf{h}_v)}
    {\sum_{z\in\mathcal{V}_i}\exp(\mathbf{u}^{\top}\mathbf{h}_z)},
    \qquad
    \mathbf{h}_{\mathcal{G}_i}
    =
    \sum_{v\in\mathcal{V}_i}\alpha_v\mathbf{h}_v.
\end{equation}
The pooling vector is shared across queries, while the weights depend on the encoded partial graph and the current-parent indicator.

\textbf{State fusion and output heads.}
We encode the query with the same frozen sentence encoder and apply a separate trainable linear projection to obtain $\mathbf{h}_q\in\mathbb{R}^{d}$.
Concatenating $\mathbf{h}_{\mathcal{G}_i}$, $\mathbf{h}_{v_i}$, and $\mathbf{h}_q$ yields the 384-dimensional state $\mathbf{s}_i$.
The branching head is a two-layer MLP with a $d$-dimensional hidden layer and a Tanh activation.
The operator-selection head linearly projects $\mathbf{s}_i$ into the operator embedding space and computes dot-product scores against the operator embedding table.
The critic uses a linear layer to map the same state to a scalar value.
The graph encoder, pooling module, and state projections are shared by the actor and critic.

%% file: appendix/semantic_similarity.tex
\subsection{Semantic Similarity}
\label{app:semantic-similarity}

The similarity function is task-dependent. For mathematical reasoning and
question answering, we match segments of the generated output $y_i$ and the
reference solution $y^*$ in both directions. Let $\{g_i\}_{i=1}^{m}$ and
$\{r_j\}_{j=1}^{n}$ denote the generated and reference segments, respectively,
and let $e(\cdot)$ be a text encoder. We instantiate $e(\cdot)$ with the
pretrained ROSCOE encoder~\citep{ROSCOE_Golovneva_2022_221207919}\footnote{\url{https://huggingface.co/facebook/roscoe-512-roberta-base}}.
We use its final-layer \texttt{[CLS]} representation with $\ell_2$ normalization
and keep the encoder fixed during policy training. Pairwise cosine similarities
are $c_{ij}=\cos\!\left(e(g_i),e(r_j)\right)$. We compute the average
best-match scores in both directions and combine them using their harmonic mean:
\begin{equation}
    P=\frac{1}{m}\sum_{i=1}^{m}\max_j c_{ij},\qquad
    R=\frac{1}{n}\sum_{j=1}^{n}\max_i c_{ij},\qquad
    \operatorname{sim}_{\mathrm{text}}(y_i,y^*)=
    \begin{cases}
        \frac{2PR}{P+R}, & P,R>0,\\
        0, & \text{otherwise}.
    \end{cases}
    \label{eq:semantic-similarity}
\end{equation}
Here, $P$ measures alignment with the reference, while $R$ measures coverage of
the reference content. We define this similarity as zero if either solution is
empty.

For code generation, we use a model-free AST-based structural similarity. We
parse each Python program into an abstract syntax tree and count the occurrences
of each node-type label, producing count vectors $\mathbf{v}(y_i)$ and
$\mathbf{v}(y^*)$. Operator names are included in the labels: for example,
binary addition is represented as \texttt{BinOp:Add}, and comparisons as
\texttt{Compare:Lt} or \texttt{Compare:Eq}. Identifiers and literal values are
not included, so programs such as $a+b$ and $x+y$ receive the same counts. The
code similarity is the cosine similarity of these vectors, clipped to $[0,1]$:
\begin{equation}
    \operatorname{sim}_{\mathrm{code}}(y_i,y^*)
    =\operatorname{clip}_{[0,1]}\!\left(
    \frac{\mathbf{v}(y_i)^\top\mathbf{v}(y^*)}
    {\lVert\mathbf{v}(y_i)\rVert_2\,\lVert\mathbf{v}(y^*)\rVert_2}
    \right).
\end{equation}
The score is set to zero if parsing fails for either program, or if the reference
program is empty. Thus, the overall $\operatorname{sim}(y_i,y^*)$ uses
$\operatorname{sim}_{\mathrm{code}}$ for code tasks and
$\operatorname{sim}_{\mathrm{text}}$ for mathematical and question-answering
tasks.

%% file: appendix/query_selection_theory.tex
\subsection{Why the Competence-Boundary Score}
\label{app:query-selection-theory}

This section explains why the acquisition score of Section~\ref{sec:aqsm} takes the form $w(q)=\tilde p_q(1-\tilde p_q)$.
We first show that $p_q(1-p_q)$ upper-bounds how much a single policy update can improve a query's success probability, and then characterize the effect of drawing $\tilde p_q$ from the posterior rather than using its mean.

\textbf{Setup.}
Fix a query $q$ and let $\tau=(\mathcal{G},\xi)$ denote a rollout, where $\mathcal{G}\sim\pi_\theta(\cdot\mid q)$ is the sampled architecture and $\xi$ collects the execution randomness of the backbone LLM, which does not depend on $\theta$.
Let $b(\tau)\in\{0,1\}$ indicate whether the final answer satisfies the task's success criterion, and write $p_q(\theta)=\mathbb{E}_\tau[b(\tau)]$ for the success probability tracked by Equation~\ref{eq:query-posterior-update}.
The construction budget bounds the depth and branching of $\mathcal{G}$ and the operator pool $\mathcal{O}$ is finite, so the trajectory space is finite; since $\mathcal{D}_{\mathrm{train}}$ is finite as well, we may set $L^2(\theta)=\max_{q'\in\mathcal{D}_{\mathrm{train}}}\mathbb{E}_\tau\|\nabla_\theta\log\pi_\theta(\mathcal{G}\mid q')\|^2<\infty$, a quantity shared by all queries at the current $\theta$.

\begin{proposition}[Improvement bound]
\label{prop:improvement-bound}
Under the setup above,
\begin{equation}
    \|\nabla_\theta\,p_q(\theta)\|\;\le\;L(\theta)\sqrt{p_q(\theta)\bigl(1-p_q(\theta)\bigr)}.
    \label{eq:gradient-bound}
\end{equation}
If in addition $p_q(\cdot)$ is $\beta$-smooth, then a gradient step $\theta'=\theta+\eta\nabla_\theta p_q(\theta)$ satisfies
\begin{equation}
    p_q(\theta')-p_q(\theta)
    \;\le\;
    \eta\Bigl(1+\tfrac{\beta\eta}{2}\Bigr)L^2(\theta)\;p_q(\theta)\bigl(1-p_q(\theta)\bigr).
    \label{eq:improvement-bound}
\end{equation}
\end{proposition}

\begin{proof}
By the score-function identity, $\nabla_\theta p_q=\mathbb{E}_\tau[b(\tau)\nabla_\theta\log\pi_\theta(\mathcal{G}\mid q)]$, where the score involves only $\mathcal{G}$ because $\xi$ does not depend on $\theta$.
Since $\mathbb{E}_\tau[\nabla_\theta\log\pi_\theta(\mathcal{G}\mid q)]=\nabla_\theta\sum_{\mathcal{G}}\pi_\theta(\mathcal{G}\mid q)=0$, subtracting the constant baseline $p_q$ leaves the gradient unchanged, giving $\nabla_\theta p_q=\mathbb{E}_\tau[(b(\tau)-p_q)\nabla_\theta\log\pi_\theta(\mathcal{G}\mid q)]$.
Applying the Cauchy--Schwarz inequality and using $\mathbb{E}_\tau[(b(\tau)-p_q)^2]=\operatorname{Var}[b(\tau)]=p_q(1-p_q)$ for a Bernoulli outcome yields Equation~\ref{eq:gradient-bound}.
For the second claim, $\beta$-smoothness gives $p_q(\theta')\le p_q(\theta)+\langle\nabla_\theta p_q,\theta'-\theta\rangle+\frac{\beta}{2}\|\theta'-\theta\|^2$; substituting $\theta'-\theta=\eta\nabla_\theta p_q$ and then Equation~\ref{eq:gradient-bound} yields Equation~\ref{eq:improvement-bound}.
\end{proof}

\textbf{Interpretation.}
The right-hand side of Equation~\ref{eq:improvement-bound} is the quantity $w(q)$ estimates, up to a factor shared by all queries at the current $\theta$.
The bound therefore states that $p_q(1-p_q)$ caps how much one update can raise a query's success probability.
A query the policy solves almost always ($p_q\to1$) or almost never ($p_q\to0$) cannot be improved much by any single update, regardless of the direction the gradient takes.
Sampling queries in proportion to $w(q)$ thus withholds budget from queries whose improvement is provably capped, which is the sense in which the competence boundary is the informative region.
We sample in proportion to $w(q)$ rather than taking its maximizer because $p_q$ is estimated rather than known, and proportional sampling retains support on every training query.


\begin{proposition}[Effect of posterior sampling]
\label{prop:posterior-sampling}
Let $n_q=S_q+F_q$ and let $\hat p_q=(1+S_q)/(2+n_q)$ denote the mean of the posterior in Equation~\ref{eq:query-success-posterior}. Then
\begin{equation}
    \mathbb{E}_{\tilde p_q}\bigl[w(q)\bigr]
    =\frac{(1+S_q)(1+F_q)}{(2+n_q)(3+n_q)}
    =\frac{n_q+2}{n_q+3}\;\hat p_q(1-\hat p_q).
    \label{eq:posterior-score}
\end{equation}
Moreover, under the update of Equation~\ref{eq:query-posterior-update}, the effective count satisfies $n_q<1/(1-\delta)$ for every query at every iteration.
\end{proposition}

\begin{proof}
For $X\sim\operatorname{Beta}(\alpha,\beta)$ we have $\mathbb{E}[X(1-X)]=\mathbb{E}[X]-\mathbb{E}[X^2]=\alpha\beta/[(\alpha+\beta)(\alpha+\beta+1)]$.
Setting $\alpha=1+S_q$ and $\beta=1+F_q$ gives the first equality, and dividing by $\hat p_q(1-\hat p_q)=(1+S_q)(1+F_q)/(2+n_q)^2$ gives the second.
Summing Equation~\ref{eq:query-posterior-update} over the two counts yields $n_q\leftarrow\delta n_q+1$, so starting from $n_q=0$ we obtain $n_q=(1-\delta^{t})/(1-\delta)<1/(1-\delta)$ after $t$ updates.
\end{proof}

\textbf{Interpretation.}
Equation~\ref{eq:posterior-score} shows that drawing $\tilde p_q$ from the posterior reproduces the plug-in score $\hat p_q(1-\hat p_q)$ in expectation, shrunk by a factor that approaches one as evidence accumulates.
What posterior sampling adds is dispersion: the score is a random variable whose spread reflects how well $p_q$ is currently estimated, so a query with little recent evidence is not deterministically excluded.
The second claim shows that the discount $\delta$ caps the effective evidence at $1/(1-\delta)$, which is $10$ at our setting $\delta=0.9$, so the posterior never collapses to a point mass and this dispersion persists throughout training.
Beyond tracking changes in policy capability, $\delta$ thus also controls how much exploration the acquisition score retains.
For illustration, at $\delta=0.9$ a competence-boundary query with substantial evidence ($S_q=F_q=4.5$) attains $\mathbb{E}[w(q)]\approx0.229$, an unobserved query attains $1/6\approx0.167$, and a query observed at length but never or always solved attains $\approx0.076$.


%% file: appendix/algorithm.tex
\subsection{ASPO Training Algorithm}
\label{app:aspo-algorithm}

\begin{algorithm}[H]
    \caption{ASPO training}
    \label{alg:aspo}
    \small
    \begin{algorithmic}[1]
        \Require Training set $\mathcal{D}_{\mathrm{train}}$, operator pool $\mathcal{O}$, epochs $T$, warm-up epochs $T_{\mathrm{w}}$, discovery budget $K$
        \Ensure Learned architectural policy $\pi_\theta$
        \State Initialize actor and critic parameters $\theta,\phi$
        \State Set $S_q=F_q=0$ and $\mathrm{solved}[q]=\mathrm{false}$ for every training query $q$
        \For{epoch $e=1,\ldots,T$}
            \State $B_{\mathrm{rem}}\gets|\mathcal{D}_{\mathrm{train}}|$
            \While{$B_{\mathrm{rem}}>0$}
                \If{$e\leq T_{\mathrm{w}}$}
                    \State Take the next batch from a traversal of $\mathcal{D}_{\mathrm{train}}$
                \Else
                    \State Draw $\tilde p_q$ from posteriors, refresh $w(q)$ (Equation~\ref{eq:query-selection-weight}), and sample a batch with weights $w(q)$
                \EndIf
                \State Limit the batch size to at most $B_{\mathrm{rem}}$ queries
                \State Freeze rollout parameters $\theta_{\mathrm{old}}\gets\theta$, $\phi_{\mathrm{old}}\gets\phi$; clear buffer $\mathcal{B}$
                \For{each $(q,y^*)$ in the batch, while $B_{\mathrm{rem}}>0$}
                    \State $d_q\gets(e>T_{\mathrm{w}})\land\neg\mathrm{solved}[q]$ \Comment{Discovery flag}
                    \State $m_q\gets K$ if $d_q$, otherwise $m_q\gets 1$
                    \State Collect $m_q$ DAG rollouts using $\pi_{\theta_{\mathrm{old}}}$ and $\mathcal{O}$ (Section~\ref{sec:mas-construction})
                    \State $B_{\mathrm{rem}}\gets B_{\mathrm{rem}}-m_q$ \Comment{Count all executed rollouts}
                    \For{each collected trajectory}
                        \State Evaluate node outputs and final-answer correctness $b_q\in\{0,1\}$
                        \State Compute $r(a_i)$, $G(a_i)$ and $\hat{A}_i$ (Equations~\ref{eq:substructure-reward}--\ref{eq:ppo-advantage-ratio})
                        \If{$\neg d_q$ or $b_q=1$ or some $\hat{A}_i>0$}
                            \State Add the trajectory, fixed PPO targets and old action probabilities to $\mathcal{B}$
                        \EndIf
                        \State If $\neg d_q$, update $S_q,F_q$ using $b_q$ (Equation~\ref{eq:query-posterior-update})
                        \State $\mathrm{solved}[q]\gets\mathrm{solved}[q]\lor(b_q=1)$
                    \EndFor
                \EndFor
                \State If $\mathcal{B}\neq\emptyset$, update actor, critic and shared encoder with PPO
            \EndWhile
        \EndFor
        \State \Return $\pi_\theta$
    \end{algorithmic}
\end{algorithm}

%% file: appendix/benchmarks.tex
\subsection{Benchmarks and Data Curation}
\label{app:benchmarks}

We evaluate on six benchmarks covering mathematical reasoning (GSM8K, MATH),
general question answering (HotpotQA, MMLU-Pro), and code generation (MBPP+, HumanEval+).
For each benchmark, we select examples from publicly available datasets and split
them into training and test sets at an approximately $1{:}2$ ratio, using the same
split for all compared methods. Sources and selection procedures are detailed below;
Table~\ref{tab:benchmarks} summarizes the sample counts.

\textbf{Selection toward harder problems.}
For GSM8K and MATH we bias selection toward the harder end of each dataset, so that the
comparison is not dominated by queries that a single LLM call already answers: we retain the
lowest-solve-rate GSM8K problems and restrict MATH to Lv.5, as specified below. Absolute
scores on these two splits are therefore well below those reported on the full benchmarks and
should not be compared with them --- in Table~\ref{tab:main}, Vanilla, a single LLM call,
reaches only $48.6\%$ on the resulting GSM8K split. The remaining four benchmarks are not
difficulty-filtered: HotpotQA and MMLU-Pro are sampled uniformly over their annotation grids,
and MBPP+ and HumanEval+ are used in full.

\textbf{Mathematical reasoning.}
\begin{itemize}[leftmargin=1.4em,topsep=2pt,itemsep=2pt]
    \item \textbf{GSM8K}~\citep{GSM8K_cobbe2021training} --- grade-school math word
    problems requiring multi-step arithmetic reasoning.%
    \footnote{\url{https://huggingface.co/datasets/openai/gsm8k}}
    \emph{Selection:} we select the $600$ problems with the lowest solve rates
    from the publicly available GSM8K difficulty annotations%
    \footnote{\url{https://huggingface.co/datasets/lime-nlp/GSM8K_Difficulty}}.
    \emph{Split:} $200$ train / $400$ test.

    \item \textbf{MATH}~\citep{MATH_hendrycks2021measuring} --- competition mathematics
    problems, annotated with $7$ subject types and $5$ difficulty levels (Lv.1--Lv.5).%
    \footnote{\url{https://huggingface.co/datasets/EleutherAI/hendrycks_math}}
    \emph{Selection:} we restrict to the hardest tier (Lv.5) and sample uniformly across
    the $7$ subject types, keeping difficulty fixed at the highest level while balancing
    subject representation. \emph{Split:} $210$ train / $490$ test.
\end{itemize}

\textbf{General question answering.}
\begin{itemize}[leftmargin=1.4em,topsep=2pt,itemsep=2pt]
    \item \textbf{HotpotQA}~\citep{HotpotQA_yang2018hotpotqa} --- multi-hop QA, with two
    question types (\emph{comparison} / \emph{bridge}) and three difficulty labels
    (\emph{easy} / \emph{medium} / \emph{hard}).%
    \footnote{\url{https://huggingface.co/datasets/hotpotqa/hotpot_qa}}
    We use the distractor setting, in which
    each question is accompanied by supporting and distractor paragraphs, so answers are
    judged against the provided passages and no additional retrieval is required.
    \emph{Selection:} we sample uniformly over the type~$\times$~difficulty grid.
    \emph{Split:} $210$ train / $420$ test.

    \item \textbf{MMLU-Pro}~\citep{MMLUPro_wang2024mmlu} --- $10$-option
    multiple-choice questions spanning $14$ subject categories (no difficulty labeling).%
    \footnote{\url{https://huggingface.co/datasets/TIGER-Lab/MMLU-Pro}}
    \emph{Selection:} uniform sampling across the $14$ categories for both splits.
    \emph{Split:} $210$ train / $420$ test.
\end{itemize}

\textbf{Code generation.}
\begin{itemize}[leftmargin=1.4em,topsep=2pt,itemsep=2pt]
    \item \textbf{MBPP+}~\citep{EvalPlus_liu2023your} --- the EvalPlus extension of
    MBPP~\citep{MBPP_austin2021program}, which augments each Python programming task with
    a substantially larger, higher-coverage test suite ($378$ problems; no
    difficulty/type labeling).%
    \footnote{\url{https://huggingface.co/datasets/evalplus/mbppplus}}
    \emph{Split:} $126$ train / $252$ test.

    \item \textbf{HumanEval+}~\citep{EvalPlus_liu2023your} --- the EvalPlus extension of
    HumanEval~\citep{HumanEval_chen2021evaluating}, with strengthened test cases
    ($164$ problems).%
    \footnote{\url{https://huggingface.co/datasets/evalplus/humanevalplus}}
    \emph{Split:} $54$ train / $110$ test.
\end{itemize}

\begin{table}[t!]
\centering
\small
\begin{tabular}{llcccl}
\toprule
\textbf{Benchmark} & \textbf{Task family} & \textbf{Train} & \textbf{Test} & \textbf{Total} & \textbf{Metric} \\
\midrule
GSM8K       & Math       & $200$ & $400$ & $600$ & Accuracy \\
MATH        & Math       & $210$ & $490$ & $700$ & Accuracy \\
HotpotQA    & General QA & $210$ & $420$ & $630$ & F1 \\
MMLU-Pro    & General QA & $210$ & $420$ & $630$ & Accuracy \\
MBPP+       & Coding     & $126$ & $252$ & $378$ & pass@1 \\
HumanEval+  & Coding     & $54$  & $110$ & $164$ & pass@1 \\
\bottomrule
\end{tabular}
\caption{Curated train/test sample counts and evaluation metrics across benchmarks.}
\label{tab:benchmarks}
\end{table}

%% file: appendix/representative_benchmarks.tex
\subsection{Representative Benchmark Selection}
\label{app:representative-benchmarks}

The analyses in Sections~\ref{sec:exp-transfer}, ~\ref{sec:exp-cost}, ~\ref{sec:ablation}, and ~\ref{sec:sensitivity} repeat training and evaluation under many configurations: different execution backbones, ablated variants of the framework, and sweeps over individual hyperparameters.
Running all six benchmarks for every configuration is expensive, so these analyses use one representative benchmark per domain, chosen so that the method ordering it induces best reproduces the ordering obtained from all six.

\textbf{Procedure.}
For each benchmark we rank the thirteen methods of Table~\ref{tab:main} by mean performance.
Averaging the six ranks of a method gives its overall rank, and the vector of overall ranks across methods is the ordering we wish to preserve.
For every benchmark we then compute the Spearman correlation between its own rank vector and this target vector, and within each domain we keep the benchmark with the higher correlation.
Table~\ref{tab:rank-corr} reports the results, selecting GSM8K for mathematical reasoning, HotpotQA for general question answering, and HumanEval+ for code generation.

\textbf{Validation.}
Taking the best benchmark within each domain is a greedy rule, so we also enumerate all $2^3=8$ ways of picking one benchmark per domain and correlate the average rank over the selected three with the overall rank over all six.
The greedy choice is also the best of the eight: GSM8K, HotpotQA and HumanEval+ reach a Spearman correlation of $0.907$, against $0.778$ to $0.891$ for the remaining seven combinations.

\begin{table}[t!]
\centering
\small
\begin{tabular}{llcc}
\toprule
\textbf{Domain} & \textbf{Benchmark} & \textbf{Spearman} & \textbf{Kendall $\tau_b$} \\
\midrule
Mathematical reasoning     & \textbf{GSM8K}      & $\mathbf{0.692}$ & $\mathbf{0.545}$ \\
                           & MATH                & $0.609$          & $0.494$ \\
\midrule
General question answering & \textbf{HotpotQA}   & $\mathbf{0.656}$ & $0.520$ \\
                           & MMLU-Pro            & $0.647$          & $0.520$ \\
\midrule
Code generation            & \textbf{HumanEval+} & $\mathbf{0.727}$ & $\mathbf{0.562}$ \\
                           & MBPP+               & $0.617$          & $0.468$ \\
\bottomrule
\end{tabular}
\caption{Rank correlation between each benchmark's method ordering and the overall ordering across all six benchmarks. \textbf{Bold} marks the benchmark selected for each domain.}
\label{tab:rank-corr}
\end{table}

%% file: appendix/baseline_setups.tex
\subsection{Baseline Setups}
\label{app:baseline-setups}

This section describes the twelve baselines of Section~\ref{sec:exp-setup}: for each, the
search space it optimizes, how that space is optimized, and the configuration we run. We
group them as in the main table --- single-agent, multi-agent collaboration, task-level MAS search, and query-level MAS search.

\textbf{Vanilla.}
A single LLM call that answers the query directly, with no scaffold; it lower-bounds the
benefit of any multi-agent structure. No training. 

\textbf{CoT~\citep{CoT_wei2022chain}.}
Chain-of-thought prompting elicits intermediate reasoning steps before the final answer
rather than answering directly. No training; the method is purely a prompt. Our setup is
identical to Vanilla except for the reasoning instruction, so the two differ only in the
prompt.

\textbf{AgentVerse~\citep{AgentVerse_chen2024agentverse}.}
A hand-designed four-stage loop --- expert recruitment, collaborative decision-making,
action execution, and evaluation --- in which a recruiter assembles a task-specific team
whose proposal is revised over several rounds of feedback. The topology is fixed by design
and is never searched, and there is no gradient training: all adaptation happens through
inference-time LLM interaction. We use the official \texttt{tasksolving} pipeline with its
default \texttt{cnt\_agents}$=2$, \texttt{max\_turn}$=3$, and
\texttt{max\_criticizing\_rounds}$=3$, with the external executor disabled.

\textbf{DyLAN~\citep{DyLAN_liu2024a}.}
Agents form a layered feed-forward network and run in two stages. \emph{Training} consists of
\emph{team optimization}: a label-free trial over a pool of candidate roles on the training
split, in which each node rates its predecessors on a $1$--$5$ scale, the ratings propagate
backward as edge weights into an Agent Importance Score, and the top-scoring agents
form one fixed team per domain. No gradients are involved, and this team is the only artifact
carried forward. \emph{Testing} runs that team alone, adding two query-dependent mechanisms:
a listwise ranker that keeps the two best answers and deactivates the rest, and early
stopping once more than $2/3$ of active agents agree.
The official pools cover three of our benchmarks --- $12$ code writers and reviewers for
HumanEval, $7$ domain experts for MMLU, and $4$ role-free agents for MATH --- so we pair each remaining benchmark with the pool
of matching task type: MBPP+, HotpotQA, and GSM8K respectively. 

\textbf{ADAS~\citep{ADAS_hu2025automated}.}
The search space is an \emph{entire agentic system expressed as Python code} --- a
\texttt{forward} function spanning prompts, tool calls, and control flow. A meta-agent,
conditioned on a growing archive of previously discovered agents, programs the next
candidate; the paper likens archive recombination to LLM-implemented crossover but defines
no explicit population or selection operator. There is no gradient training and no learned
model: each iteration designs a candidate (with two self-reflection rounds), evaluates it on
a validation split, repairs runtime errors for up to five rounds, and appends it to the
archive; validation performance is the only signal. The output is one fixed program per
dataset. We use the official implementation with $30$ generations.

\textbf{AFlow~\citep{AFlow_zhang2025aflow}.}
The search space is a \emph{code-represented workflow}: LLM-invoking nodes connected by
Python control flow, plus a pool of predefined operators. Because the model, temperature,
and output format of each node are frozen, the effective space is prompts, edges, and
operator choices. The search is a Monte-Carlo tree search whose expansion operator is an
LLM, with each tree node holding a complete workflow. Optimization is again gradient-free:
an LLM optimizer rewrites code and prompts, guided by the validation score and by textual
``experience'' propagated back up the tree, and each candidate is executed five times on
validation. The result is one fixed workflow per dataset. We run the official
implementation with $20$ rounds, $4$ samples, and one validation round. Following the
official test convention, the
best-validation round's workflow is copied into the test slot before evaluation.

\textbf{EvoMAS~\citep{EvoMASHiL_wei2026evomas}.}
The search space is the graph topology together with agent roles and prompts, evolved by
seven strategies: three exploratory (diversity expansion, conceptual recombination,
cross-domain hybridization), three exploitative (fine optimization, best-practice
synthesis, role specialization), and one custom operator. A meta-controller fires every
$K$ generations to prune and induce rules and to update the strategy distribution by
multiplicative weights. Evolution is LLM-driven with no gradients; the fitness signal is
training-split accuracy minus a cost term. We run $20$ generations with a population of
$8$ and a reflection interval of $K=5$, since generated workflows can otherwise contain unbounded loops.

\textbf{MaAS~\citep{MaAS_pmlr-v267-zhang25bi}.}
MaAS replaces the search for a single workflow with an \emph{agentic supernet}: a
distribution over layer-wise operator configurations, conditioned on the query. Two
mechanisms make the sampled architecture query-dependent --- a mixture-of-experts top-$p$
router that activates operators until their cumulative score exceeds a threshold (varying
the \emph{width} of each layer), and an early-exit operator that truncates sampling
(varying the \emph{depth}). Training is two-sided: the controller is updated by a
policy-gradient estimator with cost-aware importance weights, while the operators
themselves, which admit no numerical gradient, are revised by agent-generated
\emph{textual} gradients over their prompts, temperatures, and node structure. We use the
official implementation with its default of $4$ samples. MaAS shares ASPO's
operator-selection space.

\textbf{RAAS~\citep{RAAS_yang2026raas}.}
RAAS adopts the MaAS agentic supernet verbatim; its contribution is the \emph{evaluation
signal} rather than a new search space. For each query it samples a cohort of $N$
workflows, scores each by a trimmed mean over $K$ executions, and uses the cohort mean as
a group baseline, yielding a zero-centered advantage in the style of GRPO --- so the
supernet is trained by policy gradient without a learned critic. We run the official
implementation with its defaults of $N=5$ and $K=5$. The upstream code required several
corrections before it produced a valid run; the two consequential ones are a reward bug
(candidates were graded against a field absent from our data, so every sample scored zero
and the controller optimized cost alone) and a grouping bug (a missing key collapsed each
batch into a single advantage group, computing advantages across different problems). We
verified on the first batch of every run that rewards are non-zero and that the number of
per-query groups is correct. RAAS also shares ASPO's operator-selection space.

\textbf{Puppeteer~\citep{Puppter_NEURIPS2025_f1320d2e}.}
Puppeteer searches neither prompts nor code but the \emph{sequence of agent activations}: a
centralized orchestrator selects one agent per reasoning step, so the topology is never
represented explicitly and instead emerges incrementally, recoverable as a directed graph
only after the fact. The orchestrator is a genuinely learned neural policy trained online
with REINFORCE, rewarded by terminal task success minus a compute cost. We use the official
implementation with its defaults (episode length $4$, up to $3$ parallel explorations,
$\lambda=0.1$, $\gamma=0.99$). Puppeteer also shares ASPO's operator-selection space.

\textbf{RADAR~\citep{RADAR_zhang2026radar}.}
RADAR generates the \emph{communication topology} together with each node's role, using a
conditional discrete graph diffusion model: a redundancy-aware ordering network masks nodes
and edges, and a denoiser regenerates the graph node by node conditioned on the query, with
redundancy measured by effective size. Both networks are gradient-trained --- the denoiser
on a variational bound and the ordering network by REINFORCE, using the denoiser's
reconstruction likelihood as its reward --- and the topology is generated per query at
inference. We use the official implementation with adapters for our benchmarks, ordering
the data so that the train/test boundary falls exactly at the intended index.

\textbf{ASPEC~\citep{ASPEC_ICLR2026_9230cf96}.}
ASPEC is an LLM meta-agent method in which a meta-optimizer evolves specialist prompts
alongside the controller's architectural choices over training epochs, with specialists
subject to pruning. We use the official implementation with $2$ samples, $6$ debate rounds, $5$ self-consistency paths, and $10$ mutator turns.

%% file: appendix/hyperparameters.tex
\begin{table}[t!]
    \centering
    \small
    \begin{tabular}{lll}
    \toprule
    \textbf{Hyperparameter} & \textbf{Value} & \textbf{Sensitivity Analysis} \\
    \midrule
    \addlinespace[4pt]
    \rowcolor{gray!15}
    \multicolumn{3}{l}{\textit{Adaptive Query-Selection Mechanism (AQSM)}} \\
    \addlinespace[3pt]
    Posterior discount $\delta$                  & $0.9$  & Section~\ref{sec:sensitivity}, Figure~\ref{fig:sensitivity-delta} \\
    Discovery samples $K$                        & $3$    & Section~\ref{sec:sensitivity}, Figure~\ref{fig:sensitivity-K} \\
    \midrule
    \addlinespace[4pt]
    \rowcolor{gray!15}
    \multicolumn{3}{l}{\textit{Substructure-aware Optimization}} \\
    \addlinespace[3pt]
    Maximum number of layers                     & $5$    & Appendix~\ref{app:architecture-sensitivity}, Figure~\ref{fig:sensitivity} \\
    Maximum number of children per node           & $2$    & Appendix~\ref{app:architecture-sensitivity}, Figure~\ref{fig:sensitivity} \\
    Cost coefficient $\lambda$                   & $0.01$ & Appendix~\ref{app:architecture-sensitivity}, Figure~\ref{fig:sensitivity} \\
    Warm-up epochs $T_{\mathrm{w}}$              & $3$    & Section~\ref{sec:sensitivity}, Figure~\ref{fig:sensitivity-epochs} \\
    Training epochs $T$                          & $6$    & Section~\ref{sec:sensitivity}, Figure~\ref{fig:sensitivity-epochs} \\
    \bottomrule
    \end{tabular}
    \caption{ASPO hyperparameter settings grouped by component, with references to the corresponding sensitivity analyses.}
    \label{tab:hyperparameters}
\end{table}

\subsection{ASPO Implementation Details and Hyperparameter Reference}
\label{app:hyperparameters}

This section describes ASPO's PPO implementation and aggregation rules, and provides a reference table linking its hyperparameter settings to the corresponding sensitivity analyses.

\textbf{PPO implementation.}
We adopt the standard clipped PPO objective~\citep{PPO_schulman2017proximal} with conventional training settings, rather than introducing a new policy-update rule.
We use Adam with a learning rate of $10^{-4}$ for the actor, critic, and shared GNN.
Trainable query- and operator-embedding parameters use a learning rate of $10^{-5}$ with no weight decay.
The clipping threshold is $\epsilon_{\mathrm{clip}}=0.2$, corresponding to the probability-ratio clipping interval $[0.8,1.2]$.
Each batch contains eight queries and is used for at most four PPO updates.
We stop further updates on the batch if the absolute value of the approximate KL divergence exceeds $0.02$.
We set the subgraph discount factor in Equation~\ref{eq:substructure-return} to $\gamma=0.95$: MAS construction spans at most five layers, so a value close to one keeps the reward of a distant descendant nearly undiscounted while still preferring earlier refinements.

\textbf{Aggregation.}
We distinguish the local aggregation used to compute substructure rewards from the final aggregation used to produce the system's answer.
For $\operatorname{Agg}$ in Equation~\ref{eq:substructure-reward}, we directly use the \texttt{ScEnsemble} operator from our operator pool (Appendix~\ref{app:operator-space}) to combine the child outputs into $y_i^{+}$.
For the final aggregation in the progressive DAG rollout (Section~\ref{sec:mas-construction}), we collect the outputs of all leaf operators and apply a task-dependent rule.
For open-ended generation tasks, such as code generation, we use the same ensemble operator.
For multiple-choice and fill-in-the-blank questions, we vote over the leaf answers and select the most frequent answer.

\textbf{Hyperparameter reference.}
Table~\ref{tab:hyperparameters} summarizes the ASPO-specific settings reported in Section~\ref{sec:exp-setup} and links each to its corresponding sensitivity analysis.
Section~\ref{sec:sensitivity} examines the posterior discount $\delta$, discovery budget $K$, and training and warm-up durations.
Appendix~\ref{app:architecture-sensitivity} examines the cost coefficient $\lambda$ and the maximum numbers of layers and children per node.


%% file: appendix/cost_analysis.tex
\subsection{Cost--Performance Trade-offs by Benchmark and Phase}
\label{app:cost-analysis}

This appendix extends the cost--performance analysis in Section~\ref{sec:exp-cost} with two complementary views: training and inference expenditure on GSM8K, HotpotQA, and HumanEval+, and wall-clock time on HumanEval+.
Both are compared with final test performance to examine the monetary and time costs of the resulting systems.

\begin{figure}[t!]
    \centering
    \includegraphics[width=\linewidth]{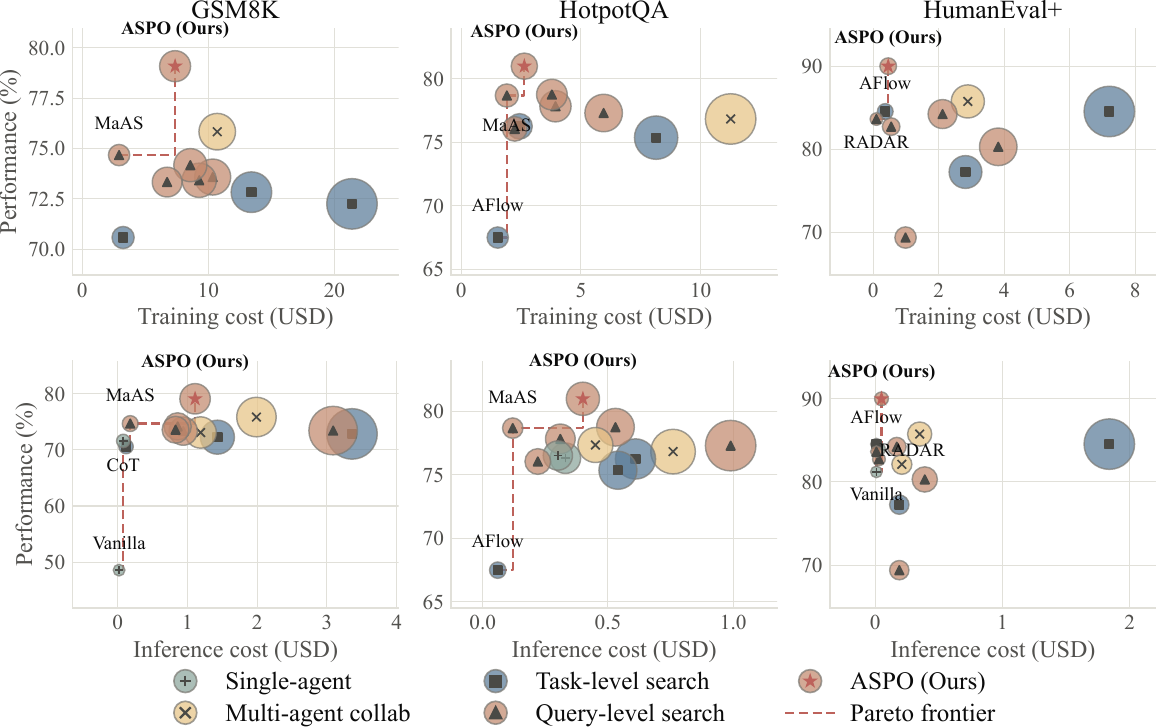}
    \caption{Cost--performance trade-offs on GSM8K, HotpotQA, and HumanEval+.
    The top and bottom rows show training and inference costs, respectively, against the same final test performance: accuracy, F1, and pass@1 for the three columns.
    Costs are in USD, with separate axis ranges and bubble-size scales for each panel.
    Dashed lines indicate the observed Pareto frontiers.
    Vanilla, CoT, and AgentVerse incur no training cost and are omitted from the top row.}
    \label{fig:cost-panels}
\end{figure}

\textbf{Monetary cost.}
Figure~\ref{fig:cost-panels} separates the aggregate comparison in Section~\ref{sec:exp-cost} by benchmark and execution phase.
Unlike the normalized costs in Figure~\ref{fig:cost-bubble}, each panel reports cumulative expenditure in USD for the corresponding benchmark and phase.
\textbf{(i) For training}, ASPO achieves the highest test performance on all three benchmarks with costs of \$7.37, \$2.64, and \$0.46, respectively.
These are lower than DyLAN's corresponding costs of \$10.71, \$11.25, and \$2.89.
The comparison with MaAS is more task-dependent: ASPO costs more to train on GSM8K and HotpotQA, but less on HumanEval+ (\$0.46 versus \$0.55), where it also improves pass@1 by $7.27$ points.
Thus, ASPO's higher performance does not require a consistently larger training expenditure.
\textbf{(ii) At inference}, ASPO costs \$1.11, \$0.40, and \$0.05 on the three benchmarks, below those of DyLAN and EvoMAS in every case.
However, MaAS is cheaper at inference on all three, costing \$0.18, \$0.12, and \$0.03, while trailing ASPO by $4.41$, $2.31$, and $7.27$ performance points.
A possible explanation is that uniform training gives substantial weight to queries already solvable by simple workflows, biasing the policy toward inexpensive architectures.
The policy may consequently under-allocate computation to harder queries, reducing inference cost at the expense of answer quality.
In conclusion, these results place ASPO at the highest-performance end of all six observed Pareto frontiers.

\textbf{Wall-clock time.}
Figure~\ref{fig:wall-time-performance} complements the monetary-cost comparison with training, test, and total wall-clock times on HumanEval+.
ASPO achieves $90.00\%$ pass@1 with $70$ minutes of training and $12$ minutes of testing ($82$ minutes in total).
This is shorter than DyLAN's $89$ minutes and EvoMAS's $199$ minutes, while improving pass@1 by $4.24$ and $5.45$ percentage points, respectively.
Faster alternatives remain available: MaAS and AFlow take $66$ and $42$ minutes in total, respectively, but fall behind ASPO by $7.27$ and $5.45$ pass@1 points.
To sum up, ASPO achieves the highest pass@1 with a total wall-clock time below the median across all evaluated methods.

\begin{figure}[t!]
    \centering
    \begin{minipage}[t]{0.535\linewidth}
        \vspace{0pt}
        \centering
        \includegraphics[width=\linewidth]{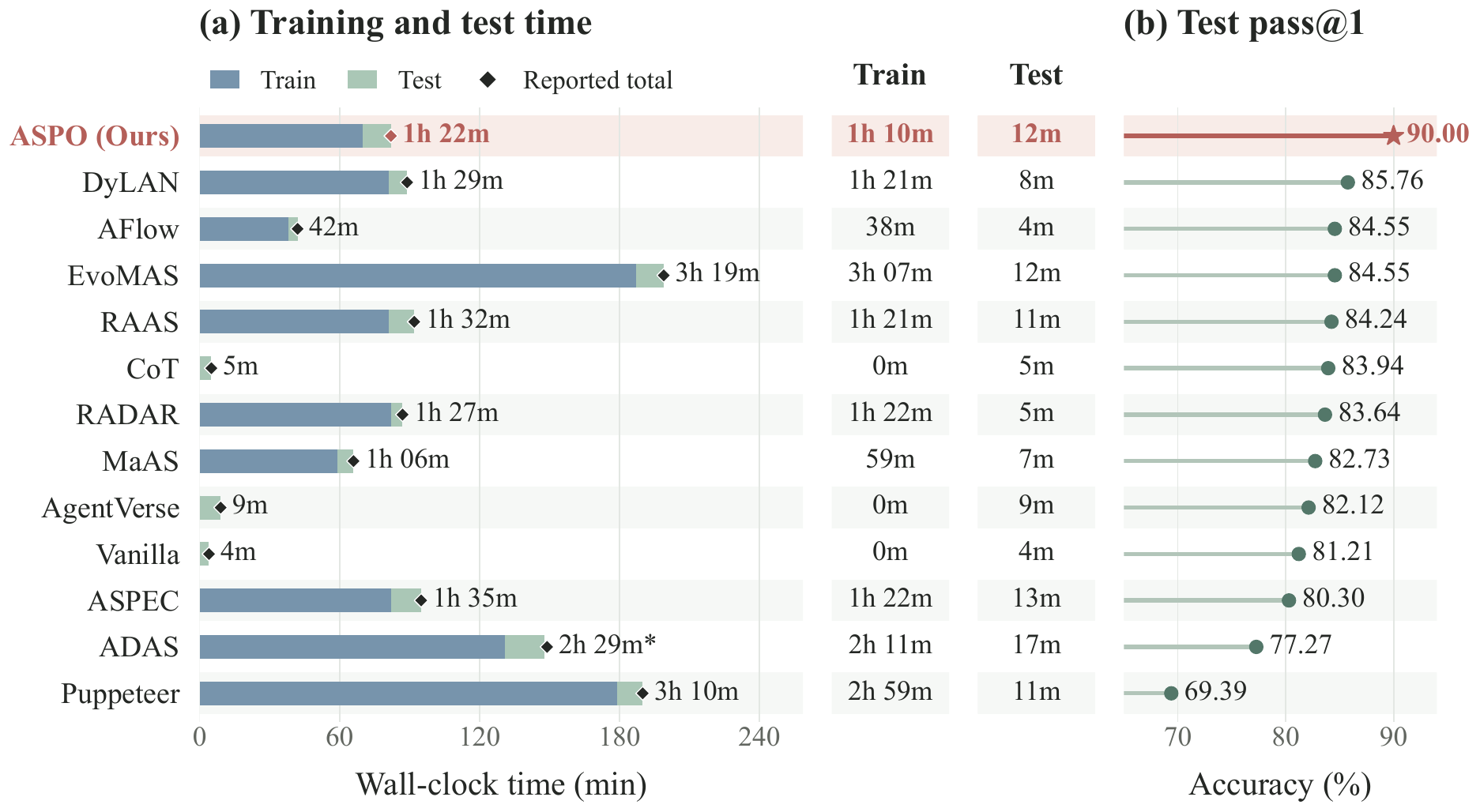}
        \caption{Wall-clock times on HumanEval+.}
        \label{fig:wall-time-performance}
    \end{minipage}\hfill
    \begin{minipage}[t]{0.46\linewidth}
        \vspace{0pt}
        \centering
        \includegraphics[width=\linewidth]{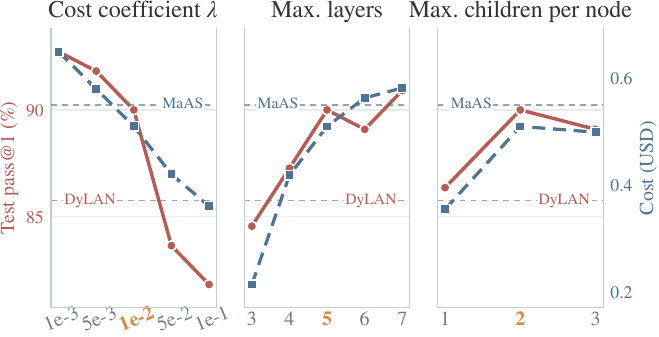}
        \caption{Sensitivity of pass@1 and total cost to the three hyperparameters that affect the searched MAS architecture, on HumanEval+.}
        \label{fig:sensitivity}
    \end{minipage}
\end{figure}

%% file: appendix/architecture_sensitivity.tex
\subsection{Sensitivity to Architecture Hyperparameters}
\label{app:architecture-sensitivity}


Figure~\ref{fig:sensitivity} examines sensitivity to three architecture-related hyperparameters on HumanEval+: \textbf{the cost coefficient $\lambda$} in Equation~\ref{eq:substructure-return}, \textbf{the maximum number of layers}, and \textbf{the maximum number of children per node}.
The defaults ($\lambda=0.01$, five layers, and two children per node) were fixed before these sensitivity experiments.
For $\lambda$, increasing the coefficient from $0.001$ to $0.01$ reduces accuracy by $2.7$ points, with a further $6.4$-point drop at $0.05$, while cost decreases from \$$0.65$ to \$$0.36$ over the full range.
This shows the performance--cost trade-off induced by stronger cost penalties.
For the layer budget, accuracy increases from $84.6$ to $90.0$ as the limit grows from three to five, then largely plateaus.
Extending the limit from five to seven layers adds $14\%$ to cost for a gain of $0.9$ points.
For the branching limit, increasing the maximum from two to three children yields similar accuracy and cost ($90.0$ at \$$0.51$ versus $89.1$ at \$$0.50$), suggesting limited sensitivity to this increase in branching capacity.

%% file: appendix/operator_space.tex
\subsection{Operator Space}
\label{app:operator-space}

We instantiate the operator pool $\mathcal{O}$ with the ten operators below. Each operator is an LLM-based operation that either \emph{generates} a candidate solution, \emph{aggregates} several candidates, \emph{refines} a parent solution, or supplies \emph{reference knowledge}. Unless noted otherwise, every solver/refiner operator emits a structured lane containing a \texttt{reasoning\_process} and a \texttt{final\_answer}; an optional \texttt{\{background\}} note (produced by the Background operator) is prepended to the prompt to aid reasoning and passed through unchanged, and parent-to-child information is threaded as separate \texttt{\{parent\_reasoning\}} and \texttt{\{parent\_answer\}} fields.

\begin{enumerate}[leftmargin=1.4em,labelsep=0.4em,topsep=2pt,itemsep=1pt]
    \item \textbf{Generate.} Produces a solution to the query following a given free-form instruction. It is the most basic generate-class operator and takes only the query (plus optional background).
    \item \textbf{GenerateCoT.} Solves the query with explicit chain-of-thought reasoning~\citep{CoT_wei2022chain}, making every intermediate deduction explicit before stating the final answer.
    \item \textbf{MultiGenerateCoT.} Samples several independent chain-of-thought solutions to the same query, producing a diverse pool of candidate lanes for downstream aggregation~\citep{SelfConsistency_wang2023selfconsistency}.
    \item \textbf{Plan.} A generate-class operator that first designs an explicit step-by-step plan~\citep{PlanAndSolve_wang2023plan} (without solving) and then executes that plan step by step to obtain the answer.
    \item \textbf{Background.} A retrieval / tool-use operator that reads only the query, extracts its core concepts, retrieves encyclopedia knowledge for them, and synthesizes solving-useful reference knowledge (definitions, formulas, identities). It does not solve the problem.
    \item \textbf{ScEnsemble.} Aggregates several candidate solutions~\citep{SelfConsistency_wang2023selfconsistency} by grouping mathematically/semantically equivalent final answers, selecting the plurality answer (breaking ties by reasoning quality), and synthesizing the selected reasoning.
    \item \textbf{SelfRefine.} Takes a parent solution and refines it~\citep{SelfRefine_madaan2023self}, finding and fixing errors or gaps and improving rigor while preserving a correct result.
    \item \textbf{Variation.} Produces a new solution that deliberately follows a different line of reasoning from the parent, exploring an alternative approach to diversify the candidate pool rather than critiquing the parent.
    \item \textbf{Debate.} Runs a sequential two-agent debate~\citep{Debate_du2024improving} (a careful defender and a skeptical challenger) over two rounds with no external judge, returning the challenger's final solution.
    \item \textbf{Programmer.} Solves the query by writing and executing Python code in a sandbox~\citep{PAL_gao2023pal,PoT_chen2023program} (retrying with error feedback), then synthesizes the final answer from the code and its execution result.
\end{enumerate}

Each operator's prompt template is shown below; \texttt{\{$\cdot$\}} denotes an input placeholder filled at execution time.

\begin{PromptBox}[taskblue]{Operator 1: Generate}
    \begin{promptsection}{promptGray}
        \textbf{[Background]}~\texttt{\{background\}}
    \end{promptsection}
    \begin{promptsection}{taskblue}
        \textbf{[Instruction]}~\texttt{\{instruction\}}\\
        \textbf{Problem:}~\texttt{\{query\}}
    \end{promptsection}
    \begin{promptsection}{warmuporange}
        \textbf{[Output]} \texttt{reasoning\_process}; \texttt{final\_answer}.
    \end{promptsection}
\end{PromptBox}

\begin{PromptBox}[taskblue]{Operator 2: GenerateCoT}
    \begin{promptsection}{promptGray}
        \textbf{[Background]}~\texttt{\{background\}}
    \end{promptsection}
    \begin{promptsection}{taskblue}
        \textbf{[Instruction]}~\texttt{\{instruction\}}\\[2pt]
        Solve this using EXPLICIT CHAIN-OF-THOUGHT reasoning: think step by step,
        making every intermediate deduction explicit and in order, and state the
        final answer ONLY after the full reasoning chain.\\
        \textbf{Problem:}~\texttt{\{query\}}
    \end{promptsection}
    \begin{promptsection}{warmuporange}
        \textbf{[Output]}
        \texttt{reasoning\_process}: the complete, self-contained step-by-step chain of thought;
        \texttt{final\_answer}: the final answer only.
    \end{promptsection}
\end{PromptBox}

\begin{PromptBox}[taskblue]{Operator 3: MultiGenerateCoT}
    \begin{promptsection}{taskblue}
        Uses the same prompt as \textbf{GenerateCoT} (Operator 2), executed as
        $N{=}3$ independent samples. The three resulting candidate lanes are passed
        on together for downstream aggregation (e.g., by ScEnsemble or Debate).
    \end{promptsection}
\end{PromptBox}

\begin{PromptBox}[taskblue]{Operator 4: Plan (two calls: design then execute)}
    \begin{promptsection}{taskblue}
        \textbf{[Call 1 --- Design]}\\
        Read the problem and design a STEP-BY-STEP PLAN to solve it. Do NOT solve
        the problem yet --- only lay out the steps. Output EXACTLY this format:
        ``\texttt{Total N steps.}'' followed by one line per step
        (``\texttt{Step k: ...}''). Each step states WHAT to do (the method/action),
        not the computation. Make the FIRST step search for the simplest route
        (identify any reusable structure or shortcut that shortens the problem);
        make the LAST step a verification step (check the result against the
        problem's conditions and constraints).\\
        \textbf{Problem:}~\texttt{\{query\}}\\
        \textbf{Output:}~\texttt{plan}.
    \end{promptsection}
    \begin{promptsection}{exploityellow}
        \textbf{[Call 2 --- Execute]}\\
        You are given the problem and a step-by-step PLAN. EXECUTE THE PLAN: carry
        out each step in order, showing the work for every step, until you reach the
        answer. Stay faithful to the plan's steps.\\
        \textbf{Problem:}~\texttt{\{query\}}\qquad \textbf{Plan:}~\texttt{\{plan\}}
    \end{promptsection}
    \begin{promptsection}{warmuporange}
        \textbf{[Output]}
        \texttt{reasoning\_process}: the full step-by-step execution that follows the
        plan; \texttt{final\_answer}: the final answer only.
    \end{promptsection}
\end{PromptBox}

\begin{PromptBox}[taskblue]{Operator 5: Background (three steps: keyword, retrieval, synthesize)}
    \begin{promptsection}{taskblue}
        \textbf{[Step 1 --- Keyword extraction]}\\
        Read the problem and identify its 1--3 CORE concepts/topics --- the things
        that, if looked up in an encyclopedia, would supply useful definitions and
        formulas for solving it. Output them as short search terms (matching an
        encyclopedia article title), comma-separated.\\
        \textbf{Problem:}~\texttt{\{query\}}\qquad \textbf{Output:}~\texttt{keywords} (1--3, comma-separated).
    \end{promptsection}
    \begin{promptsection}{promptGray}
        \textbf{[Step 2 --- Retrieval]} (executed in code, no LLM call)\\
        For each keyword, search Wikipedia and fetch the best-matching article's
        introduction; concatenate the extracts as \texttt{\{wiki\}}.
    \end{promptsection}
    \begin{promptsection}{exploityellow}
        \textbf{[Step 3 --- Synthesize]}\\
        You are compiling a BACKGROUND note of reference knowledge that will help
        someone solve the problem. You are given the problem, the extracted
        keywords, and encyclopedia extracts for those keywords
        (\texttt{\{query\}}, \texttt{\{keywords\}}, \texttt{\{wiki\}}). Write a
        concise background of the knowledge most useful for solving THIS problem
        (relevant definitions, key formulas, identities, theorems).
        \textbf{Strict rules:} do NOT solve the problem or state/compute its final
        answer; provide only general reference facts, distilled from the extracts
        and your own knowledge, ignoring irrelevant parts.
    \end{promptsection}
    \begin{promptsection}{warmuporange}
        \textbf{[Output]} \texttt{background}.
    \end{promptsection}
\end{PromptBox}

\begin{PromptBox}[taskblue]{Operator 6: ScEnsemble}
    \begin{promptsection}{promptGray}
        \textbf{[Background]}~\texttt{\{background\}}
    \end{promptsection}
    \begin{promptsection}{taskblue}
        You are combining several candidate solutions to the SAME problem; each
        candidate contains a reasoning process and a final answer. Follow this
        decision procedure:
        \begin{enumerate}[leftmargin=1.5em,nosep]
            \item Extract the final answer from every candidate; treat equivalent answers as the SAME answer.
            \item Count how many candidates support each answer. If one answer has a strict plurality/majority, select a candidate supporting it (independent agreement is strong evidence).
            \item If the highest vote is tied, do NOT break the tie by candidate order --- inspect the competing reasoning, verify it yourself, and choose the most sound one (reasoning quality takes priority over vote count).
            \item Write a self-contained synthesis of the selected answer's reasoning, correcting any clear mistake; do not invent an answer unsupported by the candidates.
        \end{enumerate}
        \textbf{Problem:}~\texttt{\{query\}}\qquad \textbf{Candidate solutions:}~\texttt{\{candidates\}}
    \end{promptsection}
    \begin{promptsection}{warmuporange}
        \textbf{[Output]}
        \texttt{reasoning\_process}: a complete comparison and synthesis (state the
        vote pattern, verify the selected derivation, explain any tie-break);
        \texttt{solution\_letter}: the single letter (A, B, C, \dots) of the selected candidate.
    \end{promptsection}
\end{PromptBox}

\begin{PromptBox}[taskblue]{Operator 7: SelfRefine}
    \begin{promptsection}{promptGray}
        \textbf{[Background]}~\texttt{\{background\}}
    \end{promptsection}
    \begin{promptsection}{taskblue}
        Your task is to REFINE a proposed solution: find and fix any errors or gaps,
        improve the rigor, and keep the result correct. You are given the problem
        plus the parent solution split into its reasoning and its final answer.\\
        \textbf{Problem:}~\texttt{\{query\}}\\
        \textbf{Parent --- reasoning:}~\texttt{\{parent\_reasoning\}}\\
        \textbf{Parent --- final answer:}~\texttt{\{parent\_answer\}}
    \end{promptsection}
    \begin{promptsection}{warmuporange}
        \textbf{[Output]}
        \texttt{reasoning\_process}: the corrected, self-contained reasoning;
        \texttt{final\_answer}: the final answer only.
    \end{promptsection}
\end{PromptBox}

\begin{PromptBox}[taskblue]{Operator 8: Variation}
    \begin{promptsection}{promptGray}
        \textbf{[Background]}~\texttt{\{background\}}
    \end{promptsection}
    \begin{promptsection}{taskblue}
        Produce a NEW solution that deliberately follows a DIFFERENT line of
        reasoning from the parent solution. The goal is to EXPLORE an alternative
        approach (to diversify the candidate pool), NOT to critique or merely
        rephrase the parent.\\
        \textbf{Problem:}~\texttt{\{query\}}\\
        \textbf{Parent --- reasoning:}~\texttt{\{parent\_reasoning\}}\qquad
        \textbf{Parent --- final answer:}~\texttt{\{parent\_answer\}}
        \begin{enumerate}[leftmargin=1.5em,nosep]
            \item Solve correctly; the final answer must be fully derivable from your reasoning. If the parent's answer is already correct, reach that SAME answer by a different route.
            \item Take a FUNDAMENTALLY different approach (e.g., a different decomposition, working backward vs.\ forward, exploiting structure vs.\ brute force). Do not paraphrase or reorder the parent's steps.
            \item Do NOT critique the parent --- this is exploration, not evaluation.
            \item Before finalizing, VERIFY your own answer against the problem's conditions.
        \end{enumerate}
        Your reasoning must be COMPLETE and self-contained (readable without seeing the parent).
    \end{promptsection}
    \begin{promptsection}{warmuporange}
        \textbf{[Output]} \texttt{reasoning\_process}; \texttt{final\_answer}.
    \end{promptsection}
\end{PromptBox}

\begin{PromptBox}[taskblue]{Operator 9: Debate (sequential two agents, two rounds, four calls)}
    \begin{promptsection}{taskblue}
        \textbf{[Round 0 --- Debater A: careful \& rigorous]}\\
        A solution has been proposed (reasoning and final answer supplied
        separately). Examine it: if correct, defend it with stronger justification
        and clearer steps; if flawed, produce a corrected solution. Re-derive from
        scratch rather than trusting the proposal; first look for a structural
        shortcut, then VERIFY by checking the result against the problem's
        conditions. If the proposal is correct, keep the answer UNCHANGED. Produce
        your OWN complete, self-contained derivation.\\
        \textbf{Inputs:}~\texttt{\{query\}}, \texttt{\{parent\_reasoning\}}, \texttt{\{parent\_answer\}}, \texttt{\{background\}}.
    \end{promptsection}
    \begin{promptsection}{exploreyellow}
        \textbf{[Round 0 --- Debater B: skeptical \& creative]}\\
        A solution has been proposed and Debater A has responded. Critically examine
        BOTH for logical flaws, unstated assumptions, edge cases, or better
        approaches. Target common failure points: a missed simplifying insight; an
        arithmetic/sign slip; an unchecked edge case or unmet constraint;
        incomplete or double-counted cases. Even when the proposal and A agree,
        independently verify --- agreement can hide a shared mistake. Produce your
        OWN complete, self-contained derivation.\\
        \textbf{Inputs:}~\texttt{\{query\}}, \texttt{\{parent\_reasoning\}}, \texttt{\{parent\_answer\}}, \texttt{\{a\_reasoning\}}, \texttt{\{a\_answer\}}, \texttt{\{background\}}.
    \end{promptsection}
    \begin{promptsection}{exploityellow}
        \textbf{[Round 1 --- Debater \{role\} (A then B): final round]}\\
        Read your opponent's latest response: (1) are there valid points that should
        change your answer? (2) flaws to rebut? (3) can you strengthen your own
        solution? Then produce your final, definitive solution and verify it against
        the problem's conditions. If your previous answer is already correct, keep it
        UNCHANGED.\\
        \textbf{Inputs:}~\texttt{\{query\}}, \texttt{\{own\_reasoning\}}, \texttt{\{own\_answer\}}, \texttt{\{opp\_reasoning\}}, \texttt{\{opp\_answer\}}, \texttt{\{background\}}.
    \end{promptsection}
    \begin{promptsection}{warmuporange}
        \textbf{[Output]} \texttt{reasoning\_process}; \texttt{final\_answer} (the operator returns Debater B's final-round output).
    \end{promptsection}
\end{PromptBox}

\begin{PromptBox}[taskblue]{Operator 10: Programmer (two calls: code then synthesize)}
    \begin{promptsection}{promptGray}
        \textbf{[Background]}~\texttt{\{background\}}
    \end{promptsection}
    \begin{promptsection}{taskblue}
        \textbf{[Call 1 --- Code generation]}\\
        Solve the problem by writing EXACT Python. Computing with code avoids the
        slips that ruin long hand derivations, so lean on it for heavy or
        error-prone calculation. Classify the problem and pick the right tools
        (prefer exact/symbolic computation); use the background to set up the
        computation correctly; the parent reasoning/answer are REFERENCE ONLY (may
        be wrong; re-derive the setup).
        \textbf{Rules:} define \texttt{def solve()} taking NO arguments and RETURNING
        the result; return EXACT values; no \texttt{print}/plots/file/network I/O;
        give EVERY loop a finite termination bound so it finishes well under the
        sandbox time limit.\\
        \textbf{Problem:}~\texttt{\{query\}}\\
        \textbf{Parent reasoning (reference only):}~\texttt{\{parent\_reasoning\}}\qquad
        \textbf{Parent final answer (reference only):}~\texttt{\{parent\_answer\}}\\
        \textbf{Feedback (on a failed run):}~\texttt{\{feedback\}} --- the previous
        code and its error are fed back so the model can fix it (up to three tries).\\
        \textbf{Output:}~\texttt{code} (complete, self-contained, ready to run), then executed in a sandbox.
    \end{promptsection}
    \begin{promptsection}{exploityellow}
        \textbf{[Call 2 --- Synthesis]}\\
        You solved the problem by writing and running Python. Given the problem, the
        code, and the execution result
        (\texttt{\{query\}}, \texttt{\{code\}}, \texttt{\{status\}}, \texttt{\{run\_output\}}),
        turn this into the final answer:
        \texttt{reasoning\_process} explains the approach the code took and interprets
        the result; \texttt{final\_answer} is DERIVED from the execution result in the
        exact form asked. If the status is an error, or the result is empty or clearly
        inconsistent, do NOT trust it --- re-reason using the background.
    \end{promptsection}
    \begin{promptsection}{warmuporange}
        \textbf{[Output]} \texttt{reasoning\_process}; \texttt{final\_answer}.
    \end{promptsection}
\end{PromptBox}